\documentclass[pdflatex,sn-mathphys]{sn-jnl}% Math and Physical Sciences Reference Style

\usepackage{amsmath,amssymb,amsthm}
\usepackage{bm}
\usepackage{mathtools}
\DeclarePairedDelimiter{\abs}{\lvert}{\rvert}
\jyear{2021}%

\theoremstyle{thmstyleone}%
\newtheorem{theorem}{Theorem}%  meant for continuous numbers
\newtheorem{proposition}[theorem]{Proposition}% 
\newtheorem{lemma}[theorem]{Lemma}% 
\newtheorem{corollary}[theorem]{Corollary}% 

\theoremstyle{thmstyletwo}%
\newtheorem{example}{Example}%
\newtheorem{remark}{Remark}%

\theoremstyle{thmstylethree}%
\newtheorem{definition}{Definition}%

\def\Rbb{\mathbb{R}}
\def\1{\bm{1}}
\def\0{\bm{0}}
\def\u{{\bm u}}
\def\f{{\bm f}}
\def\g{{\bm g}}

\def\x{{\bm x}}
\def\z{{\bm z}}
\def\<{\langle}
\def\>{\rangle}
\def\rmd{\mathrm{d}}
\def\dmu{\mathrm{d}\mu}
\def\bmu{{\bm\mu}}
\def\sign{{\rm sign}}

\begin{document}

\title[Robust Estimation for Kernel Exponential Families]{
Robust Estimation for Kernel Exponential Families with Smoothed Total Variation Distances
%Article Title
}


\author*[1,2]{\fnm{Takafumi} \sur{Kanamori}}%\email{kanamori@c.titech.ac.jp}

\author[1]{\fnm{Kodai} \sur{Yokoyama}}%\email{yokoyama.k.ap@m.titech.ac.jp}

\author[1]{\fnm{Takayuki} \sur{Kawashima}}%\email{kawashima@c.titech.ac.jp}

\affil*[1]{\orgname{Science Tokyo}}

\affil[2]{\orgname{RIKEN}}

%\affil[3]{\orgdiv{Department}, \orgname{Organization}, \orgaddress{\street{Street}, \city{City}, \postcode{610101}, \state{State}, \country{Country}}}


\abstract{
In statistical inference, we commonly assume that samples are independent and identically distributed from a probability distribution included in a pre-specified statistical model.
However, such an assumption is often violated in practice. Even an unexpected extreme sample called an {\it outlier} can significantly impact classical estimators.
Robust statistics studies how to construct reliable statistical methods that efficiently work even when the ideal assumption is violated.
Recently, some works revealed that robust estimators such as Tukey's median are well approximated by the generative adversarial net (GAN), a popular learning method for complex generative models using neural networks.
GAN is regarded as a learning method using integral probability metrics (IPM), which is a discrepancy measure for probability distributions. 
In most theoretical analyses of Tukey's median and its GAN-based approximation, however, 
the Gaussian or elliptical distribution is assumed as the statistical model.
In this paper, we explore the application of GAN-like estimators to a general class of statistical models.
As the statistical model, we consider the kernel exponential family that includes both finite and infinite-dimensional models.
To construct a robust estimator, we propose the smoothed total variation (STV) distance as a class of IPMs.
Then, we theoretically investigate the robustness properties of the STV-based estimators. 
Our analysis reveals that the STV-based estimator is robust against the distribution contamination 
for the kernel exponential family. 
Furthermore, we analyze the prediction accuracy of a Monte Carlo approximation method, which circumvents the computational difficulty of the normalization constant. 
}




\maketitle

\section{Introduction}
\label{sec:Introduction}

In statistical inference, we often assume ideal assumptions for data distribution. 
For instance, samples are independent and identically distributed from a probability distribution that is included in a pre-specified statistical model. 
However, such an assumption is often violated in practice. 
The data set may contain unexpected samples by, say, the failure of observation equipment. 
Even a single extreme sample called an {\it outlier} can have a large impact to classical estimators. 
Robust statistics studies how to construct reliable statistical methods that efficiently work 
even when ideal assumptions on the observation are violated. 

\subsection{Background}
Robust statistics has a long history. The theoretical foundation of robust statistics was established by many works
including~\cite{huber64:_robus,tukey75:_mathem,hampel11:_robus_statis,donoho83}. 
A typical problem setting of robust statistics is that the data is observed from the contaminated model 
$(1-\varepsilon)P+\varepsilon Q$, where $Q$ is a contamination distribution and $\varepsilon$ is a contamination ratio. 
The purpose is to estimate the target distribution $P$ or its statistic, such as 
the location parameter or variance from the data. 
Samples from $Q$ are regarded as outliers. 
Sensitivity measures such as the influence function or breakdown point are 
used to evaluate the robustness of the estimator against outliers. 
The influence function is defined as the G\^{a}teaux derivative of the estimation functional 
to the direction of the point mass at an outlier. 
The most B-robust estimator is the estimator that minimizes the 
supremum of the modulus of the influence function~\citep{hampel11:_robus_statis}. 
The median is the most B-robust estimator for the mean value of the one-dimensional Gaussian distribution. 
The breakdown point is defined by the maximum contamination ratio such that the estimator still 
gives meaningful information of the target distribution $P$. 
It is well-known that the median asymptotically attains the maximum breakdown point under some assumptions. 
For the above reason, the median estimator is regarded as the most robust estimator of the mean value 
especially when the data is distributed from a contaminated Gaussian distribution. 

Numerous works have appeared that study multi-dimensional extension of the median. 
The componentwise median and geometric median are straightforward extensions 
of the median to multivariate data. 
These estimators are, however, suboptimal as the estimator of the location parameter under the contaminated 
model~\citep{chen18:_robus,diakonikolas16:_robus_estim_high_dimen_comput_intrac,lai16:_agnos_estim_mean_covar}. 
The authors of \cite{tukey75:_mathem} proposed the concept of data depth, which measures how deeply embedded a point is in the scattered data. 
The deepest point for the observed data set is called Tukey's median. Numerous works have proved that Tukey's median has desirable properties as an estimator of the multivariate location parameter for the target distribution; 
high breakdown point~\citep{donoho92:_break_proper_locat_estim_based},
redescending property of the influence function~\citep{chen02:_tukey},
the min-max optimal convergence rate~\citep{chen18:_robus}, 
and possession of general notions of statistical depth~\citep{zuo00:_gener}. 

As another approach of robust statistics, minimum distance methods have been explored by many authors 
including~\cite{yatracos85:_rates_kolmog,basu10:_statis_infer,a.98:_robus_effic_estim_minim_densit_power_diver,jones01:_compar,fujisawa08:_robus,kanamori15:_robus}. 
The robustness of statistical inference is closely related to the topological structure 
induced by the loss function over the set of probability distributions~\citep{csiszar67}. 
For instance, the negative log-likelihood loss used in the maximum likelihood estimator (MLE) corresponds 
to the minimization of the Kullback-Leibler (KL) divergence from the empirical distribution 
to the statistical model. 
A deep insight by \cite{csiszar67} revealed that the KL-divergence does not yield a relevant topological
structure on the set of probability distributions. 
That is a reason that the MLE does not possess robustness property to contamination. 
To construct a robust estimator, the KL divergence is replaced with the other divergences 
that induce a relevant topology, such as the density-power divergence or pseudo-spherical 
divergence~\citep{a.98:_robus_effic_estim_minim_densit_power_diver,jones01:_compar,fujisawa08:_robus}, 
These divergences are included in the class of Bregman divergence~\citep{bregman67:_relax_method_of_findin_commonc}.
The robustness property of Bregman divergence has been 
investigated by many works, including~\cite{murata04:_infor_geomet_u_boost_bregm_diver,kanamoriar:_affin_invar_diver_compos_scores_applicy}. 

Besides Bregman divergence, f-divergence~\citep{JRSS-B:Ali+Silvey:1966,csiszar67} and integral probability metrics (IPMs) \citep{muller97:_integ} have been widely applied to statistics and machine learning. 
In statistical learning with deep neural networks, 
generative adversarial networks (GANs) have been proposed to estimate complex data distribution, 
such as the image-generating process~\cite{goodfellow14:_gener_adver_nets}. 
GAN is regarded as a learning algorithm that minimizes the Jensen-Shannon (JS) divergence. 
The original GAN is extended to f-GAN, which is the learning method using f-divergence as a loss function~\cite{nowozin16:_gan}. 
This extension has enabled us to use other f-divergence for the estimation of generative models. 
%Such a general class of GAN is called f-GAN. 
Note that the variational formulation of the f-divergence accepts the direct substitution of the empirical distribution. 
IPMs such as the Wasserstein distance or maximum mean discrepancy (MMD), too, are used to learn generative models~\citep{gulrajani17:_improv_train_wasser_gans,li17:_mmd_gan,cherief-abdellatif21:_finit_mmd}. 
Learning methods using f-divergence and IPMs are formulated as the min-max optimization problem. 
The inner maximization yields the variational expression of the divergence. 

Recently, \cite{gao19:_robus_estim_via_gener_adver_networ,DBLP:journals/jmlr/GaoYZ20} found a relationship 
between f-GANs and robust estimators based on the data depth. 
Roughly, their works showed that depth-based estimators for the location parameter and the covariance matrix are approximately represented by the f-GAN using the total variation (TV) distance, i.e., TV-GAN. 
Furthermore, \cite{gao19:_robus_estim_via_gener_adver_networ} proved that 
the approximate Tukey's median by the TV-GAN attains the min-max optimal convergence rate as well as the original Tukey's median~\citep{chen18:_robus}. 
The approximation of the data depth by GAN-based estimators is advantageous for computation. 
\cite{diakonikolas16:_robus_estim_high_dimen_comput_intrac,lai16:_agnos_estim_mean_covar}
proposed polynomial-time computable robust estimators of the Gaussian mean and covariance. 
However, the computation algorithm often requires knowledge such as the contamination ratio, which is usually not usable in practice. 
On the other hand, GAN-based robust estimators are computationally efficient, 
though rigorous computation cost has not been theoretically guaranteed. 
Inspired by \cite{gao19:_robus_estim_via_gener_adver_networ,DBLP:journals/jmlr/GaoYZ20}, 
some authors have investigated depth-based estimators from the standpoint of 
GANs~\citep{liu21:_robus_w_gan_based_estim,%
wu20:_minim_optim_gans_robus_mean_estim,%
zhu22:_robus_estim_nonpar_famil_gener_adver_networ,%
banghua22:_gener,zhu2019deconstructing}.

In this paper, we explore the application of GAN-based robust estimators to a general class of statistical models. 
In most theoretical analyses of Tukey's median and its GAN-based approximation, 
usually, the Gaussian distribution or elliptical distribution is assumed. 
We propose a class of IPMs called the smoothed total variation distance to construct a robust estimator. 
As the statistical model, 
we consider the {\it kernel exponential family}~\citep{fukumizu09:_algeb_geomet_statis,JMLR:v18:16-011}. 
The kernel method is commonly used 
in statistics~\citep{book:Schoelkopf+Smola:2002,mohri18:_found_machin_learn,Shalev-Shwartz:2014:UML:2621980}. 
A kernel function corresponds to a reproducing kernel Hilbert space (RKHS). 
The kernel exponential family with a finite-dimensional RKHS produces the standard exponential family, and 
an infinite-dimensional RKHS provides an infinite-dimensional exponential family. 
We propose a statistical learning method using smoothed TV distance for the kernel exponential family 
and analyze its statistical properties, such as the convergence rate under contaminated models. 
Often, the computation of the normalization constant is infeasible~\citep{gutmann12:_noise_contr_estim_unnor_statis,gutmann11:_bregm,%
GeyerC1994Otco,uehara20:_unified_statis_effic_estim_framew_unnor_model}. 
In this paper, we use a Monte Carlo approximation
to overcome the computational difficulty of the normalization constant. Then, we analyze the robustness property of approximate estimators.

\subsection{Related Works}
\label{sec:Related_Works}
This section discusses related works to our paper, highlighting GAN-based robust estimators, kernel
exponential family and unnormalized models.

\subsubsection*{Relation between Tukey's Median and GAN}
Recent studies revealed some statistical properties of depth-based estimators, including Tukey's 
median~\citep{chen18:_robus,gao19:_robus_estim_via_gener_adver_networ,%
DBLP:journals/jmlr/GaoYZ20,liu21:_robus_w_gan_based_estim,%
wu20:_minim_optim_gans_robus_mean_estim,%
zhu2019deconstructing,zhu22:_robus_estim_nonpar_famil_gener_adver_networ,banghua22:_gener}. 
A connection between the depth-based robust estimator and GAN 
was first studied by~\cite{gao19:_robus_estim_via_gener_adver_networ}. 
Originally, GAN was proposed as a learning algorithm for generative models. 
There are mainly two approaches to constructing learning algorithms for generative models. 

One is to use f-divergences called f-GAN, including vanilla  GAN~\cite{nowozin16:_gan,goodfellow14:_gener_adver_nets}.
In the learning using f-GAN, the estimator is obtained by minimizing an approximate f-divergence between the
empirical distribution and the generative model. 
In~\cite{gao19:_robus_estim_via_gener_adver_networ,DBLP:journals/jmlr/GaoYZ20,wu20:_minim_optim_gans_robus_mean_estim}, 
the robustness of f-GAN implemented by DNNs has been revealed mainly for the location parameter and variance-covariance matrix of the normal distribution.

The other approach is to use IPM-based methods such as \cite{gulrajani17:_improv_train_wasser_gans,arjovsky17:_wasser_gener_adver_networ}. 
In IPM-based learning, the estimator is computed by minimizing the gap between generalized moments of the empirical distribution and those of the estimated distribution. 
Hence, the IPM-based learning is similar to the generalized method of moments~\cite{hall2005generalized}. 
Gao,~et~al.~\cite{gao19:_robus_estim_via_gener_adver_networ} proved that Tukey's median and depth-based robust estimators are expressed as the minimization of a modified IPM between the empirical distribution and the model. 
Some works, including 
\cite{gao19:_robus_estim_via_gener_adver_networ,zhu22:_robus_estim_nonpar_famil_gener_adver_networ,liu21:_robus_w_gan_based_estim} show that also the IPM-based estimators provide robust estimators for the parameters of the normal distribution. 
\cite{liu21:_robus_w_gan_based_estim} studied the estimation accuracy of the robust estimator defined by the Wasserstein distance or the neural net distance~\citep{arora17:_gener_equil_gener_adver_nets_gans}, which is an example of the IPM.

The connection between the robust estimator and GAN allows us to apply recent developments in deep neural networks (DNNs) to robust statistics. 
However, most theoretical works on GAN-based robust statistics have focused on estimating the mean vector or variance-covariance matrix in the multivariate normal distribution, as shown above. 

In our paper, we propose a smoothed variant of total variation (TV) distance called the smoothed TV (STV) distance and investigate the convergence property for general statistical models. The STV distance class is included in IPMs. 
The most related work to our paper is \cite{zhu22:_robus_estim_nonpar_famil_gener_adver_networ}, in which the smoothed Kolmogorov-Smirnov (KS) distance over the set of probability distributions is proposed. 
In their paper, the smoothed KS distance is employed to estimate the mean value or the second-order moment of population distribution under contamination. 
On the other hand, we use the STV distance for the kernel exponential family, including infinite-dimensional models, to estimate the probability density. 
Furthermore, the estimation accuracy of the estimator using STV distance is evaluated by the TV distance. 
In many theoretical analyses, the loss function used to compute the estimator is again used to evaluate estimation accuracy. 
For example, the estimator based on the f-GAN is assessed by the same f-divergence. 
A comparison of estimation accuracy is not fairly performed in such an assessment. 
We use the TV distance to assess estimation accuracy for any STV-based learning in a unified manner. 
For that purpose, the difference between the TV distance and the STV distance is analyzed to evaluate the estimation bias induced by the STV distance.

\subsubsection*{Statistical Inference with Kernel Exponential Family}
The exponential family plays a central role in mathematical statistics~\citep{LehmCase98,AmariNagaoka00}. 
Indeed, the exponential family has rich mathematical properties from various viewpoints; 
minimum sufficiency, statistical efficiency, 
attaining maximum entropy under the moment constraint of sufficient statics, 
geometric flatness in the sense of information geometry, and so on. 

The kernel exponential family is an infinite dimensional extension of 
the exponential family~\citep{fukumizu09:_algeb_geomet_statis}. 
Due to the computational tractability, the kernel exponential family is used as a statistical model 
for non-parametric density estimation~\citep{JMLR:v18:16-011,dai2019kernel,dai2020exponential}. 

Some authors have studied robustness for infinite-dimensional 
non-parametric estimators. In \cite{chen16:_huber}, the authors study an estimator based on a robust test, which is computationally infeasible. The TV distance is used as the discrepancy measure. 
In \cite{uppal20:_robus_densit_estim_besov_ipm_losses}, the estimation accuracy of the wavelet thresholding estimator is evaluated using IPM loss defined by the Besov space. 

On the other hand, we study the robustness property of the STV-based estimator with kernel exponential family. The estimation accuracy is 
 evaluated by the TV distance, STV distance, and the norm in the parameter space. 
For the finite-dimensional multivariate normal distribution, 
we derive the minimax convergence rate of the robust estimator for the covariance matrix in terms of the Frobenius norm, 
while existing works mainly focus on the convergence rate 
under the operator norm.

\subsubsection*{Unnormalized Models}
For complex probability density models, including kernel exponential family, 
the computation of the normalization constant is often infeasible. 
In such a case, the direct application of the MLE is not possible. 
There are numerous works to mitigate the computational difficulty in statistical
inference~\citep{BesagJulian1975SAoN,HyvarinenAapo2007Seos,gutmann12:_noise_contr_estim_unnor_statis,gutmann11:_bregm,parry12:_proper_local_scorin_rules,dawid12:_proper_local_scorin_rules_discr_sampl_spaces,VarinCristiano2011AOOC,LindsayBruceG.2011IASI,GeyerC1994Otco,uehara20:_unified_statis_effic_estim_framew_unnor_model}. 
The kernel exponential family has the same obstacle. 
To avoid the computation of the normalization constant, 
\cite{JMLR:v18:16-011} investigated the estimation method using Fisher divergence. 
Another approach is to use the dual expression of the MLE for 
the kernel exponential family~\citep{JMLR:v18:16-011,dai2019kernel,dai2020exponential}. 
The IPM-based robust estimator considered in this paper, too, has a computational issue. 

Here, we employ the Monte Carlo method. 
For the infinite-dimensional kernel exponential family, the representer theorem does not work to reduce
the infinite-dimensional optimization problem to a finite one. 
However, the finite sampling by the Monte Carlo method enables us to use the standard representer theorem 
to compute the estimator. 
We analyze the relation between the number of Monte Carlo samples and the convergence rate under contaminated distribution.

\subsubsection*{Organization}
The paper is organized as follows. 
In Section~\ref{sec:Smoothed-TV-dist}, we introduce some notations used in this paper and discrepancy measures that are used in statistical inference. 
Let us introduce the relationship between the data depth and GAN-based robust estimator. 
In Section~\ref{sec:STV}, we show the definition of the smoothed TV distance. 
Some theoretical properties, such as the difference from the TV distance, are investigated. 
In Section~\ref{sec:PredictionErrorBounds_Finite-dim_RKHS}, 
the smoothed TV distance with regularization is applied for the estimation of the kernel exponential family. We derive the convergence rate of the estimator in terms of the TV distance. 
In Section~\ref{sec:Accuracy_Par_est}, 
we investigate the convergence rate of the estimator in the parameter space. 
For the estimation of the covariance matrix for the multivariate normal distribution, 
we prove that our method attains the minimax optimal rate in terms of the Frobenius norm under contamination, while most past works derived the convergence rate in terms of the operator norm. 
In Section~\ref{Approximation_and_Algorithm}, 
we investigate the convergence rate of the learning algorithm using Monte Carlo approximation for the normalization constant. 
Section~\ref{Conclusion} is devoted to the conclusion and future works. Detailed proofs of theorems are presented in Appendix.

\section{Discrepancy Measures for Statistical Inference}
\label{sec:Smoothed-TV-dist}

In this section, let us define the notations used throughout this paper. 
Then, we introduce some discrepancy measures for statistical inference. 

\subsection{Notations and Definitions}
We use the following notations throughout the paper. 
Let $\mathcal{P}$ be the set of all probability measures on a Borel space $(\mathcal{X},\mathcal{B})$, 
where $\mathcal{B}$ is a Borel $\sigma$-algebra. % and $\mu$ is a Borel measure on $\mathcal{B}$. 
For the Borel space, $L_0$ denotes the set of all measurable functions and 
$L_1(\subset L_0)$ denotes the set of all integrable functions. 
Here, functions in $L_0$ are allowed to take $\pm\infty$. 
For a function set $\mathcal{U}\subset{L_0}$ and $c\in\mathbb{R}$, 
let $c\,\mathcal{U}$ be $\{c u : u\in\mathcal{U}\}$ and $-\mathcal{U}$ be $(-1)\mathcal{U}$. 
The expectation of $f\in L_0$ for the probability distribution $P\in\mathcal{P}$ is denoted by 
$\mathbb{E}_P[f]=\int_{\mathcal{X}}f\mathrm{d}P$. 
We also write $\mathbb{E}_{X\sim{P}}[f(X)]$ or $\mathbb{E}_{{P}}[f(X)]$ to specify the random variable $X$. 
The range of the integration, $\mathcal{X}$, is dropped if there is no confusion. 
For a Borel measure $\mu$, $P\ll \mu$ denotes that $P$ is dominated by $\mu$, i.e., 
$\mu(A)=0$ for $A\in\mathcal{B}$ leads to $P(A)=0$. 
When $P\ll \mu$ holds, $P$ has the probability density function $p$ such that 
the expectation is computed by $\int_{\mathcal{X}}f(x)p(x)\mathrm{d}\mu(x)$ or $\int_{\mathcal{X}}f p\mathrm{d}\mu$. The function $p$ is denoted by $\frac{\rmd{P}}{\rmd{\mu}}$. 
The simple function on $A\in\mathcal{B}$ is denoted by $\1_A(x)$ that takes $1$ for $x\in{A}$ and $0$ for $x\not\in{A}$. 
In particular, the step function on the set of non-negative real numbers, $\mathbb{R}_+$, is denoted by $\1$ for simplicity, 
i.e., $\1(x)=1$ for $x\geq 0$ and $0$ otherwise. 
The indicator function $\mathbb{I}[A]$ takes $1$ (resp. $0$) when the proposition $A$ is true (resp. false). 
The function $\mathrm{id}:\mathbb{R}\rightarrow\mathbb{R}$ stands for the identity function, 
i.e., $\mathrm{id}(x)=x$ for $x\in\mathbb{R}$. The Euclidean norm of $\x\in\Rbb^d$ is denoted by $\|\x\|_2=\sqrt{\x^T\x}$. 

Let us define $\mathcal{H}$ as a reproducing kernel Hilbert space (RKHS) on $\mathcal{X}$. 
For $f,g\in\mathcal{H}$, the inner product on $\mathcal{H}$ is expressed by $\<f,g\>$ 
and the norm of $f$ is defined by  $\|f\|=\sqrt{\<f,f\>}$. 
For a positive number $U$, let $\mathcal{H}_U$ 
be the subset of the RKHS $\mathcal{H}$ defined by $\{f\in\mathcal{H}:\|f\|\leq U\}$. 
See \cite{berlinet2011reproducing} for details of RKHS. 

In statistical inference, discrepancy measures for probability distributions play an important role. 
One of the most popular discrepancy measures in statistics is the Kullback-Leibler (KL) divergence
\begin{align*}
 \mathrm{KL}(P,Q):=\mathbb{E}_P\!\left[\log\frac{\rmd{P}}{\rmd{Q}}\right]. 
\end{align*}
When $P\ll \mu$ and $Q\ll \mu$ holds, $P$ (resp $Q$) has the probability density $p$ (resp. $q$) for $\mu$. 
Then, the KL divergence is expressed by 
\begin{align*}
 \mathrm{KL}(P,Q)=\int p(x) \log\frac{p(x)}{q(x)} \dmu(x). 
\end{align*}
Note that the KL divergence does not satisfy the definition of the distance. Indeed, the symmetric property does not hold. 
Another important discrepancy measure is the total variation (TV) distance for $P,Q\in\mathcal{P}$, 
\begin{align}
 \label{eqn:TV-dist}
 \mathrm{TV}(P,Q)
 :=
 \sup_{A\in\mathcal{B}} \abs{\mathbb{E}_{P}[\1_A]-\mathbb{E}_{Q}[\1_A]}
 =
 \sup_{\substack{f\in L_0\\ 0\leq f\leq 1}} \mathbb{E}_{P}[f]-\mathbb{E}_{Q}[f]. 
\end{align}
When $P$ and $Q$ respectively have the probability density function $p$ and $q$ for the Borel measure $\mu$, we have 
\begin{align*}
 \mathrm{TV}(P,Q)=\frac{1}{2}\int\abs{p-q}\rmd{\mu}. 
\end{align*}
As an expansion of the total variation distance, 
the integral probability measure (IPM) is defined by 
\begin{align*}
 \mathrm{IPM}(P,Q;\mathcal{F}):=\sup_{f\in\mathcal{F}}
 \abs{\mathbb{E}_P[f]-\mathbb{E}_Q[f]},
 \end{align*}
where $\mathcal{F}\subset L_0$ is a uniformly bounded function set, i.e., 
$\sup_{f\in\mathcal{F}, x\in\mathcal{X}}\abs{f(x)}<\infty$~\citep{muller97:_integ}. 
If $\mathcal{F}$ satisfies $\mathcal{F}=-\mathcal{F}$, 
clearly one can drop the modulus sign in the definition of the IPM. 
The same property holds 
when $\mathcal{F}=\{c-f\,:\,f\in \mathcal{F}\}$ holds for a fixed real number $c$; see \eqref{eqn:TV-dist}. 
IPM includes not only the total variation distance, 
but also Wasserstein distance, Dudley distance, maximum mean discrepancy (MMD)~\citep{Bioinformatics:Borgwardt+etal:2006}, etc. 
The MMD is used for non-parametric two-sample test~\citep{JMLR:v13:gretton12a} 
and the Wasserstein distance is used for the transfer learning or the estimation with generative 
models~\citep{arjovsky17:_wasser_gener_adver_networ,shen18:_wasser_distan_guided_repres_learn_domain_adapt,gao16:_distr_robus_stoch_optim_wasser_distan,lee18:_minim_statis_learn_wasser,courty17:_optim_trans_domain_adapt}. 
In the sequel sections, we study robust statistical inference using a class of the IPM.

\subsection{Depth-based methods and IPMs}
Tukey's median is a multivariate analog of the median. Given $d$-dimensional i.i.d. samples $X_1,\ldots,X_n$, 
Tukey's median is defined as the minimum solution of 
\begin{align*}
 \min_{\bm{\mu}\in\Rbb^d}\max_{\substack{\u\in\Rbb^d\\ \|\u\|_2=1}}
 \frac{1}{n}\sum_{i=1}^{n}\mathbb{I}[\u^T(X_i-\bm{\mu})\geq 0]. 
\end{align*}
Let $\widehat{P}_{\bm{\mu}}$ be the probability distribution having the uniform point mass 
on $X_1-\bm{\mu},\ldots,X_n-\bm{\mu}$
and $\mathcal{F}$ be the function set 
$\mathcal{F}=\{\x\mapsto\mathbb{I}[\u^T\x\geq0]\,:\,\u\in\Rbb^d,\,\|\u\|_2=1\}
\cup
\{\x\mapsto \mathbb{I}[\u^T\x>0]\,:\,\u\in\Rbb^d,\,\|\u\|_2=1\}$. Then, we can confirm that 
the IPM between 
$\widehat{P}_{\bm{\mu}}$ and the multivariate standard normal distribution $N_d(\0,I_d)$
with the above $\mathcal{F}$ yields
\begin{align*}
\mathrm{IPM}(\widehat{P}_{\bm{\mu}},N_d(\0,I_d);\mathcal{F})
&= 
 \max_{\substack{\u\in\Rbb^d\\ \|\u\|_2=1}}
\frac{1}{n}\sum_{i=1}^{n}\mathbb{I}[\u^T(X_i-\bm{\mu})\geq0]-\frac{1}{2}, 
\end{align*}
i.e., one can drop the modulus in the definition of the IPM. 
Tukey's median is expressed by the minimization of the above IPM. 

Likewise, the covariance matrix depth is expressed by the IPM between 
the probability distribution $\widehat{P}_\Sigma$ having the uniform point mass 
on $\Sigma^{-1/2}X_1, \ldots, \Sigma^{-1/2}X_n$ 
for a positive definite matrix $\Sigma$ and $N_d(\0,I_d)$. 
For the function set 
$\mathcal{F}=\{\x\mapsto\mathbb{I}[\u^T(\x\x^T-I_d)\u\leq 0]\,:\,\u\in\Rbb^d\setminus\{\0\}\}$, 
we can confirm that 
\begin{align*}
&\phantom{=} \mathrm{IPM}(\widehat{P}_{\Sigma},N_d(\0,I_d);\mathcal{F})\\
&= 
 \max_{\u\in\Rbb^d\setminus\{\0\}}
 \bigg\lvert
 \frac{1}{n}\sum_{i=1}^{n}\mathbb{I}[\u^T\Sigma^{-1/2}(X_iX_i^T-\Sigma)\Sigma^{-1/2}\u\leq 0]\\
&\phantom{ \max_{\u\in\Rbb^d\setminus\{\0\}}} 
\qquad -\mathbb{P}_{Z\sim N(0,1)}(\|\u\|^2Z^2-\|\u\|^2\leq0)
 \bigg\rvert \\
&= 
 \max_{\substack{\u\in\Rbb^d\\ \|\u\|=1}}
 \bigg\lvert
 \frac{1}{n}\sum_{i=1}^{n}\mathbb{I}[(\u^T X_i)^2\leq \u^T\Sigma\u]
 -
 \mathbb{P}_{Z\sim N(0,1)}(Z^2\leq 1)
 \bigg\rvert. 
\end{align*}
The last line is nothing but the covariance matrix depth. 
The minimizer of 
$\mathrm{IPM}(\widehat{P}_{\Sigma},N_d(\0,I_d);\mathcal{F})$ over the positive definite matrix $\Sigma$
is equal to the covariance matrix estimator with the data depth. 

Another IPM-based expression of the robust estimator is presented by \cite{gao19:_robus_estim_via_gener_adver_networ}. 
In order to express the estimator by the minimization of the IPM from the empirical distribution of data $P_n$ 
to the statistical model, a variant of IPM loss is introduced.  %\cite{gao19:_robus_estim_via_gener_adver_networ} 
For instance, 
the robust covariance estimator with the data depth is expressed by 
the minimum solution of $\lim_{r\rightarrow0}\mathrm{IPM}(P_n, N_d(\0,\Sigma);\mathcal{F}_{N_d(\0,\Sigma),r})$ 
with respect to the parameter of the covariance matrix $\Sigma$, where 
$\mathcal{F}_{Q,r}$ is a function set depending on the distribution $Q$ and 
a positive real parameter $r$. 
Details are shown in Proposition 2.1 of \cite{gao19:_robus_estim_via_gener_adver_networ}. 
Though the connection between the data depth and IPM 
%presented by \cite{gao19:_robus_estim_via_gener_adver_networ} 
is not straightforward, 
the GAN-based estimator is thought to be a promising method for robust density estimation.

\section{Smoothed Total Variation Distance}
\label{sec:STV}

We define the smoothed total variation distance as a class of the IPM, and investigate its theoretical properties. All the proofs of theorems in this section are deferred to Appendix~\ref{app:proof-modulus}.

\subsection{Definition and Some Properties}
As an extension of the TV distance, let us define the smoothed  TV distance, which is a class of IPM. 
\begin{definition}
 Let $\sigma:\Rbb\rightarrow\Rbb$ be a measurable function and 
 $\mathcal{U}\subset L_0$ be a function set including the zero function. 
 For $P,Q\in\mathcal{P}$, the smoothed total variation (STV) distance, $\mathrm{STV}_{\mathcal{U},\sigma}(P,Q)$, is defined by 
 \begin{align*}
  \mathrm{STV}_{\mathcal{U},\sigma}(P,Q) := \sup_{\substack{u\in\mathcal{U}}, b\in\Rbb}
  \big\lvert\mathbb{E}_{X\sim P}[\sigma(u(X)-b)]-\mathbb{E}_{X\sim Q}[\sigma(u(X)-b)]\big\rvert. 
 \end{align*} 
\end{definition}
For the bias $b$, one can impose the constraint such as $\abs{b}\leq R$ with a positive constant $R$. 
Regardless of the constraint on the bias $b$, 
we use the notations $\mathrm{STV}(P,Q)$ and $\mathrm{STV}_\sigma(P,Q)$ 
for $\mathrm{STV}_{\mathcal{U},\sigma}(P,Q)$ if there is no confusion. 

The STV distance is nothing but the IPM with the function class 
$\{\sigma(u-b)\,:\,u\in\mathcal{U}, b\in\Rbb\}$. 
We show that the STV distance is a smoothed variant of the TV distance and shares some statistical properties. 
When the function set $\mathcal{U}$ is properly defined, and $\sigma$ is smooth, 
it is possible to construct a computationally tractable learning algorithm using the STV distance. 
On the other hand, learning with the TV distance is often computationally intractable, 
as the indicator function, which is non-differentiable, prevents from efficient optimization. 
In our paper, we focus on the STV distance such that the function in $\mathcal{U}$ is expressed by 
a ball in the RKHS. The details are shown in Section~\ref{sec:Smoothed-TVdist_KernelExp}. 

Some examples of STV distance are shown below. 
\begin{example}
 The total variation distance is expressed by $\mathrm{STV}_{L_0,\1}$. 
 \end{example}
\begin{example}
 The STV distance with the identity function $\sigma=\mathrm{id}$ 
 is reduced to the IPM defined by the function set~$\mathcal{U}$. 
 The MMD~\citep{JMLR:v13:gretton12a} is expressed by the STV distance 
 with $\mathcal{U}=\mathcal{H}_1$ and $\sigma=\mathrm{id}$. 
 The Wasserstein distance with 1st moment corresponds to 
 the STV distance with $\mathcal{U}=\{f:\mathcal{X}\rightarrow\mathbb{R}\,:\,\text{$f$ is 1-Lipschitz continuous}\}$ and $\sigma=\mathrm{id}$. 
\end{example}
The authors of \cite{liu21:_robus_w_gan_based_estim} revealed the robustness of the Wasserstein-based estimator called Wasserstein-GAN. 
As for the robustness of the MMD-based method, i.e., MMD-GAN, the authors of 
\cite{wu20:_minim_optim_gans_robus_mean_estim} obtained negative results according to theoretical analysis under a strong assumption and numerical experiments. 
Though the STV distance with RKHSs is similar to MMD, a significant difference is that a non-linear function $\sigma$ and the RKHS ball $\mathcal{H}_U$ with variable radius are used in the STV distance. 
As a result, the STV-based method recovers the robustness. 
Section~\ref{sec:PredictionErrorBounds_Finite-dim_RKHS} and thereafter will discuss this in more detail.
\begin{example}
 Let $\mathcal{U}\subset{L_0}$ be a function set. 
 The STV distance with $\mathcal{U}$ and $\sigma=\1$ is 
 the generalized Kolmogorov-Smirnov distance~\citep{zhu22:_robus_estim_nonpar_famil_gener_adver_networ}, 
 \begin{align*}
  \mathrm{STV}_{\mathcal{U},\1}(P,Q)
  &=
  \sup_{u\in\mathcal{U}, b\in\Rbb} \mathbb{E}_P[\1(u(X)-b\geq 0)]-\mathbb{E}_Q[\1(u(X)-b\geq 0)] \\ 
  &=
  \sup_{u\in\mathcal{U}, b\in\Rbb}P(u(X)\geq b)-Q(u(X)\geq b). 
 \end{align*}
 When $\sigma$ is a cumulative distribution function of a probability distribution, 
 the STV distance, $\mathrm{STV}_\sigma$, 
is the smoothed generalized Kolmogorov-Smirnov distance~\citep{zhu22:_robus_estim_nonpar_famil_gener_adver_networ}. 
\end{example}

Let us consider some basic properties of the STV distance. 
\begin{lemma}
For the STV distance, the non-negativity, $\mathrm{STV}(P,Q)\geq 0$, 
and the triangle inequality, $\mathrm{STV}(P,Q)\leq\mathrm{STV}(P,R)+\mathrm{STV}(R,Q)$
hold for $P,Q,R\in\mathcal{P}$. 
When $\sigma$ satisfies $0\leq \sigma\leq 1$, 
the inequality $\mathrm{STV}_{\sigma}(P,Q)\leq \mathrm{TV}(P,Q)$ holds.  
\end{lemma}
We omit the proof, since it is straightforward. 

Let us consider the following assumptions. 
\begin{description}
 \item[Assumption~(A)]: The function set $\mathcal{U}\subset L_0$ satisfies 
	    $\mathcal{U}=-\mathcal{U}$, i.e., $\mathcal{U}$ is closed for negation. 
 \item[Assumption~(B)]: 
	    The function $\sigma(z)$ is continuous and strictly monotone increasing. 
	    In addition, 
	    $\displaystyle\lim_{z\rightarrow-\infty}\sigma(z)=0, \lim_{z\rightarrow\infty}\sigma(z)=1$ and 
	    $\sigma(z)+\sigma(-z)=1, z\in\Rbb$ hold. 
\end{description}
Under Assumption~(B), the STV distance is regarded as 
a class of the smoothed generalized Kolmogorov-Smirnov distance~\citep{zhu22:_robus_estim_nonpar_famil_gener_adver_networ}. 

We show some properties of the STV distance under the above assumptions. 
\begin{lemma}
 \label{lemma-sigma0}
 Under Assumption~(A), the following equality holds, 
 \begin{align*}
  \mathrm{STV}_{\mathcal{U},\1}(P,Q) = \sup_{\substack{u\in\mathcal{U}, b\in\Rbb}}
  \mathbb{E}_{P}[\1[u(X)-b\geq0]]-\mathbb{E}_{Q}[\1[u(X)-b\geq0]]
\end{align*}
 for $P,Q\in\mathcal{P}$, i.e., one can drop the modulus sign. 
\end{lemma}
\begin{lemma}
\label{lemma:smoothed-TV}
 Under Assumptions~(A) and (B), it holds that 
 \begin{align*}
 \mathrm{STV}_{\mathcal{U},\sigma}(P,Q) = \sup_{\substack{u\in\mathcal{U}, b\in\Rbb}}
  \mathbb{E}_{P}[\sigma(u(X)-b)]-\mathbb{E}_{Q}[\sigma(u(X)-b)]
\end{align*}
for $P,Q\in\mathcal{P}$. 
\end{lemma}

Let us consider the STV distance such that $\mathcal{U}$ is given by the RKHS $\mathcal{H}$ or its subset $\mathcal{H}_U$. 
When the RKHS is dense in the set of all continuous functions on $\mathcal{X}$ for the supremum norm, 
the RKHS is called universal RKHS~\citep{steinwart2008support}. It is well known that the Gaussian kernel induces a universal RKHS. 
\begin{lemma}
\label{lemma:RKHS_smoothed-TV}
 Suppose that $\mathcal{H}$ be a universal RKHS. 
 Under Assumptions~(A) and (B), 
 $\mathrm{STV}_{\mathcal{H},\sigma}$ equals the TV distance. 
 Furthermore, 
 for $\mathcal{H}_U=\{f\in\mathcal{H}\,:\,\|f\|\leq U\}$, 
 the equality  
 \begin{align*}
  \lim_{U\rightarrow\infty}\mathrm{STV}_{\mathcal{H}_U,\sigma}(P,Q)
 =
 \mathrm{STV}_{\mathcal{H},\sigma}(P,Q) =\mathrm{TV}(P,Q). 
 \end{align*}
 holds for $P,Q\in\mathcal{P}$. 
\end{lemma}

\subsection{Gap between STV distance and TV distance}
\label{subsec:Difference_STV_TV}

One can quantitatively evaluate the difference between the TV distance and STV distance. 
First of all, let us define the decay rate of the function $\sigma$. 
\begin{definition}
Let $\sigma$ be the function satisfying Assumption~(B). If there exists a function $\lambda(c)$ such that 
$\lim_{c\rightarrow\infty}\lambda(c)=0$ and 
\begin{align}
 \label{eqn:decay_rate_sigma}
 \sup_{t\geq 1}\sigma(-c\log t)(t-1) \leq \lambda(c)
\end{align}
hold for arbitrary $c> C_0>0$, 
$\lambda(c)$ is called the decay rate of $\sigma$, where $C_0$ is a positive constant. 
\end{definition}
\begin{proposition}
 \label{prop:TV-STV_bias_bound}
 Assume the Assumptions~(A) and (B). Suppose that the decay rate of $\sigma$ is $\lambda(c)$ for $c>C_0>0$, 
 i.e., $\eqref{eqn:decay_rate_sigma}$ holds. 
 For $P,Q\in\mathcal{P}$, let $\mu$ be a Borel measure such that $P\ll \mu$ and $Q\ll \mu$ hold. 
 Let us define the function $s(x)$ on $\mathcal{X}$ by 
 $s(x) = \log\frac{\frac{\rmd{P}}{\rmd{\mu}}(x)}{\frac{\rmd{Q}}{\rmd{\mu}}(x)}$, 
 where $\log\frac{a}{0}=\infty$ and $\log\frac{0}{a}=-\infty$ for $a>0$ and $\frac{0}{0}=1$ by convention. 
 Suppose $s\in\mathcal{U}\oplus\mathbb{R}$. Then, for the STV distance with $c\,\mathcal{U}, c>0$ and $\sigma$, the inequality 
 \begin{align*}
  0\leq \mathrm{TV}(P,Q)-\mathrm{STV}_{c\,\mathcal{U},\sigma}(P,Q) \leq \lambda(c)(1-\mathrm{TV}(P,Q))
 \end{align*}
 holds for $c> C_0$. 
\end{proposition}
Note that for any pair of probability distributions, $P$ and $Q$, there exists a measure $\mu$ such that $P\ll\mu$ and $Q\ll\mu$. 
A simple example is $\mu=P+Q$. 
The above proposition holds only for $P$ and $Q$ such that $s\in\mathcal{U}$. 
Under mild assumptions, any pair of probability distributions in a statistical model satisfies the condition $s\in\mathcal{U}$. 

\begin{remark}
 Let us consider the case of $\mathrm{TV}(P,Q)=1$. 
 A typical example is the pair of $P$ and $Q$ for which 
 there exists a subset $A$ such that $P(A)=1$ and $Q(A)=0$. In this case, $s(x)=\infty$ (resp. $s(x)=-\infty$) for $x\in{A}$ 
 (resp. $x\not\in {A}$). If $\mathcal{U}$ includes such a function, we have $\sigma(s(x)-b)=+1$ for $x\in A$
 and otherwise $\sigma(s(x)-b)=0$. Hence, the STV distance matches with the TV distance for $P$ and $Q$. 
\end{remark}

Below, we show the lower bound of the decay rate and some examples. 
\begin{lemma}
 \label{lemma:lower_bound_DecayRate}
 Under Assumption~(B), the decay rate satisfies 
 \begin{align*}
  \displaystyle \liminf_{c\rightarrow\infty}c \lambda(c)>0. 
 \end{align*}
\end{lemma}
The above lemma means that the order of the decay rate $\lambda(c)$ is greater than or equal to $1/c$. 

\begin{example}
 For the sigmoid function $\sigma(z)=1/(1+e^{-z}), z\in\Rbb$, the decay rate is given by $\lambda(c)=1/c$ for $c>1$. 
 Indeed, the inequality
\begin{align*}
 \sigma(-c \log t)(t-1) = \frac{t-1}{1+t^c} \leq \frac{t-1}{t^c} \leq \frac{1}{c}\left(\frac{c-1}{c}\right)^{c-1}
 \leq \frac{1}{c}
\end{align*} 
holds for $t>1$ and $c>1$. We see that the sigmoid function attains the lowest order of the decay rate. 
Likewise, we find that the decay rate $\lambda(c)$ of the function $\sigma(-z) \asymp e^{-z}$ for $z>0$ 
is of the order $1/c$. 
\end{example}

\begin{example}
 For the function $\sigma(-z)\asymp z^{-\beta} (z\rightarrow\infty)$ with a constant $\beta>0$, 
 we can prove that there is no finite decay rate. Indeed, for $\sigma(-z)\asymp z^{-\beta}$, we have 
 \begin{align*}
  \sup_{t\geq1}\sigma(-c\log{t})(t-1) \asymp \frac{1}{c^\beta}\sup_{t\geq 1}\frac{t-1}{(\log{t})^\beta}=\infty.  
 \end{align*}
 In the proof of Proposition~\ref{prop:TV-STV_bias_bound}, the density ratios, $p/q$ and $q/p$, are replaced
 with the variable $t$. When the density ratios, $p/q$ and $q/p$, are both bounded above by a constant $T_0>0$, 
 the range of the supremum is restricted to $1\leq t\leq T_0$. In such a case, 
 the decay rate is $\lambda(c)=\frac{1}{c^\beta}\sup_{1\leq t\leq T_0}\frac{t-1}{(\log{t})^\beta}\asymp 1/c^\beta$. 
 The above additional assumption makes the lower bound in Lemma~\ref{lemma:lower_bound_DecayRate} smaller. 
\end{example}

\subsection{STV distances on Kernel Exponential Family}
\label{sec:Smoothed-TVdist_KernelExp}
We use the STV distance for the probability density estimation. 
There are numerous models of probability densities. 
In this paper, we focus on the exponential family and its kernel-based extension called 
\emph{kernel exponential family}~\citep{fukumizu09:_algeb_geomet_statis,JMLR:v18:16-011}. 
The exponential family includes important statistical models. 
The kernel exponential family is a natural extension of the finite-dimensional exponential family to an infinite-dimensional one while preserving computational feasibility. 
We consider the robust estimator based on the STV distance for the kernel exponential family. 

Let $\mathcal{H}$ be the RKHS endowed with the kernel function $k$. The kernel exponential family
$\mathcal{P}_\mathcal{H}$ is defined by 
\begin{align}
\label{eqn:kernel-exp-family}
 \mathcal{P}_\mathcal{H} := \left\{P_f =  p_f\rmd\mu=\exp(f-A(f))\rmd\mu\,:\,f\in\mathcal{H},\, 
 \int_{\mathcal{X}}e^{f(x)}\mathrm{d}\mu<\infty \right\}, 
\end{align}
where $A(f)$ is the moment generating function, $A(f)=\log\int_{\mathcal{X}}e^{f(x)}\mathrm{d}\mu$, 
and $\mu$ is a Borel measure on $\mathcal{X}$. 

The following lemmas indicate the basic properties of the STV distance for the kernel exponential family. 
The proofs are shown in Appendix~\ref{app:proof-modulus}. %\ref{appendix:proof-STV-KE}. 

\begin{lemma}
 \label{lemma:kexp_TV}
 For $P_f, P_g\in \mathcal{P}_\mathcal{H}$, $\mathrm{STV}_{\mathcal{H}_{U},\1}(P_f, P_g)$ 
 equals the TV distance $\mathrm{TV}(P_f, P_g)$ for any $U>0$. 
\end{lemma}
\begin{lemma}
\label{lemma:limitSTV-KEF_TV}
Suppose Assumption~(B). For $P_f, P_g\in\mathcal{P}_{\mathcal{H}}$, we have
\begin{align*}
 \lim_{U\rightarrow\infty}\mathrm{STV}_{\mathcal{H}_U,\sigma}(P_f,P_g) = \mathrm{TV}(P_f,P_g). 
\end{align*}
The convergence rate is uniformly given by 
 \begin{align*}
  0\leq \mathrm{TV}(P_f,P_g) - \mathrm{STV}_{\mathcal{H}_U,\sigma}(P_f,P_g)\leq 
  \lambda\left(\frac{U}{\|f-g\|}\right)(1-\mathrm{TV}(P_f,P_g))
 \end{align*}
 for any $P_f,P_g\in\mathcal{P}_{\mathcal{H}}$ as long as $U/\|f-g\|>C_0$, where $\lambda(c),c>C_0$ is the decay rate of 
$\sigma$. 
\end{lemma}
Lemma~\ref{lemma:RKHS_smoothed-TV} shows 
a similar result to the former part in the above lemma. 
In Lemma~\ref{lemma:limitSTV-KEF_TV}, the RKHS is not necessarity universal. 

When $\sigma$ is the sigmoid function, we have $\lambda(c)=1/c$ for $c>1$. Thus, 
 \begin{align}
  \label{eqn:diff_TV_STV_sigmoid}
  0\leq \mathrm{TV}(P_f,P_g) - \mathrm{STV}_{\mathcal{H}_U,\sigma}(P_f,P_g)
  \leq 
  \frac{\|f-g\|}{U}  (1-\mathrm{TV}(P_f,P_g))
  \leq 
  \frac{\|f-g\|}{U}
 \end{align}
 holds for $U>\|f-g\|$. 

\begin{remark}
 One can use distinct RKHSs for the STV and the probability model. 
 Let $\mathcal{H}$ and $\widetilde{\mathcal{H}}$ be RKHSs such that $\mathcal{H}\subset\widetilde{\mathcal{H}}$. 
 Then, for $P_f, P_g\in\mathcal{P}_{\mathcal{H}}$, it holds that 
 $\mathrm{STV}_{\widetilde{\mathcal{H}}_{U},\1}(P_f,P_g) = \mathrm{TV}(P_f,P_g)$ and 
 $\lim_{U\rightarrow\infty}\mathrm{STV}_{\widetilde{\mathcal{H}}_{U},\sigma}(P_f,P_g)= \mathrm{TV}(P_f,P_g)$
 under Assumption~(B) for $\sigma$. 
\end{remark}

Using the above lemma, one can approximate the learning with the TV distance 
by learning using STV distance, which does not include the non-differentiable indicator function in the loss function.

%@book{Hampel_etal86,

\section{Estiamtion with STV distance for Kernel Exponential Family}
\label{sec:PredictionErrorBounds_Finite-dim_RKHS}
We consider the estimation of the probability density using the model $\mathcal{P}_{\mathcal{H}}$. 
We assume that i.i.d. samples generated from a contaminated distribution of $P_{f_0}$. For example, the Huber contamination is expressed by the mixture
of $P_{f_0}$ and outlier distribution $Q$: 
\begin{align}
 \label{eqn:contam}
X_1,\ldots,X_n\sim (1-\varepsilon)P_{f_0}+\varepsilon Q,
\end{align}
where $Q\in\mathcal{P}$ is an arbitrary distribution. 
Our target is to estimate $P_{f_0}\in\mathcal{P}_{\mathcal{H}}$ from the data. 
For that purpose, there are numerous estimators~
\citep{%
gao19:_robus_estim_via_gener_adver_networ,%
DBLP:journals/jmlr/GaoYZ20,%
Huber.Wiley.ea1981Robuststatistics,%
liu21:_robus_w_gan_based_estim,%
tukey75:_mathem,%
wu20:_minim_optim_gans_robus_mean_estim,%
zhu2019deconstructing,
banghua22:_gener}. %
In our paper, we assume the following condition for the contaminated distribution. 
\begin{description}
 \item[Assumption~(C)]: 
    For the target distribution $P_{f_0}\in\mathcal{P}_{\mathcal{H}}$ and 
    the contamination rate $\varepsilon\in(0,1)$, 
    the contaminated distribution $P_\varepsilon$ satisfies 
    $\mathrm{TV}(P_{f_0}, P_\varepsilon)<\varepsilon$, and 
    i.i.d. samples, $X_1,\ldots,X_n$, are generated from $P_\varepsilon$. 
\end{description}
The Huber contamination~\eqref{eqn:contam} satisfies Assumption~(C). 

The model $\mathcal{P}_{\mathcal{H}}$ is regarded as a working hypothesis that explains the data generation process.
Suppose that, under an ideal situation, the data is generated from the "true distribution" $P_0$, which may not be included in $\mathcal{P}_{\mathcal{H}}$. If data contamination occurs, $P_0$ may shift to the contaminated distribution $P_{\varepsilon}$. 
On the assumption that $\min_{P\in\mathcal{P}_{\mathcal{H}}}\mathrm{TV}(P,P_0)<\varepsilon$ and $\mathrm{TV}(P_0,P_{\varepsilon}) < \varepsilon$, 
all the theoretical findings from this Section hold 
for the distribution $P_{f_0}\in\mathcal{P}_{\mathcal{H}}$ satisfying $\mathrm{TV}(P_{f_0},P_0) < \varepsilon$. 
The above discussion means that we do not need to assume that the model $\mathcal{P}_{\mathcal{H}}$ should exactly include the true distribution. 
As a working hypothesis, however, we still need to select an appropriate model $\mathcal{P}_{\mathcal{H}}$. 
For finite-dimensional models, robust model-selection methods have been studied  by~\cite{ronchetti1997robustness,sugasawa2021selection}.
This paper does not address developing practical methods for robust model selection.

In the sequel, we consider the estimator using the STV distance, which is called STV learning. All proofs of theorems in this section are deferred to Appendix~\ref{app:proof-learning}.

\subsection{Minimum STV distance Estimators}

For the RKHSs $\mathcal{H}$ and $\widetilde{\mathcal{H}}$  
such that $\mathcal{H}\subset\widetilde{\mathcal{H}}$ and 
$\dim\mathcal{H}=d\leq\widetilde{d}=\dim\widetilde{\mathcal{H}}<\infty$, 
let us consider the STV distance to obtain a robust estimator, 
\begin{align}
 \label{eqn:kernel-Tukey}
 \inf_{P_f\in \mathcal{P}_{\mathcal{H}}}\mathrm{STV}_{\widetilde{\mathcal{H}}_U, \1}(P_f, P_n)\ \longrightarrow\ \widehat{f}, 
\end{align}
where $P_n$ is the empirical distribution of data. 
Note that $\mathrm{STV}_{\widetilde{\mathcal{H}}_U,\1}(P_f, P_n)$ does not necessarily match with the TV
distance between $P_f$ and $P_n$, since
usually $P_n\in\mathcal{P}_{\mathcal{H}}$ does not hold. 
Finding  $\widehat{f}$ is not computationally feasible. 
Here, let us consider the estimation accuracy of $\widehat{f}$. 
The TV distance between the target distribution and the estimated one is evaluated as follows. 
\begin{theorem}
 \label{thm:convergence-rate}
 Assume Assumption~(C). 
 Suppose that the dimensnion $\widetilde{d}=\dim\widetilde{\mathcal{H}}$ is finite. 
 The TV distance between the target distribution and the estimator $P_{\widehat{f}}$ given by \eqref{eqn:kernel-Tukey} 
 is bounded above by 
\begin{align*}
 \mathrm{TV}(P_{f_0},P_{\widehat{f}})
 \lesssim
 \varepsilon + \sqrt{\frac{\widetilde{d}}{n}} + \sqrt{\frac{\log(1/\delta)}{n}}
\end{align*}
with probability greater than $1-\delta$. 
\end{theorem}
The STV-learning with $\widetilde{\mathcal{H}}=\mathcal{H}$ %$\mathrm{STV}_{\mathcal{H},\1}$ 
attains the lower error bound with $\sqrt{d/n}$ instead of $\sqrt{\widetilde{d}/n}$. 
In what follows, we assume $\mathcal{H}=\widetilde{\mathcal{H}}$. 

Next, let us consider the estimator using $\mathrm{STV}_{\mathcal{H}_U,\sigma}$, 
where $\sigma$ is the function satisfying Assumption~(B): %the condition of Lemma~\ref{lemma:smoothed-TV}: 
\begin{align}
\label{eqn:STV-learn}
\min_{f\in\mathcal{H}} \mathrm{STV}_{\mathcal{H}_U,\sigma}(P_f,P_n) \longrightarrow  P_{\widehat{f}}.  
\end{align}
A typical example of $\sigma$ is the sigmoid function, which 
leads to a computationally tractable learning algorithm. 
The optimization problem is written by the min-max problem,
\begin{align*}
 \min_{f\in\mathcal{H}}\max_{u\in\mathcal{H}_U, b\in\Rbb} 
\mathbb{E}_{P_f}[\sigma(u(X)-b)]-\mathbb{E}_{P_n}[\sigma(u(X)-b)]. 
\end{align*}
When $\mathcal{H}$ is the finite-dimensional RKHS, 
the objective function 
$\mathbb{E}_{P_f}[\sigma(u(X)-b)]-\mathbb{E}_{P_n}[\sigma(u(X)-b)]$
is differentiable w.r.t. the finite-dimensional parameters of the model under the mild assumption. 

In the same way as Theorem~\ref{thm:convergence-rate}, we can see that the estimator~\eqref{eqn:STV-learn} satisfies 
\begin{align}
 \mathrm{TV}(P_{f_0}, P_{\widehat{f}}) 
 &= 
 \underbrace{
 \mathrm{TV}(P_{f_0}, P_{\widehat{f}}) - 
 \mathrm{STV}_{\mathcal{H}_U,\sigma}(P_{f_0}, P_{\widehat{f}})}_{\text{bias}}
 + \mathrm{STV}_{\mathcal{H}_U,\sigma}(P_{f_0}, P_{\widehat{f}}) \nonumber\\
&\lesssim
 \text{bias} + \varepsilon+\sqrt{\frac{d}{n}} \label{eqn:Bias_term_abstract}
\end{align}
with a high probability. We need to control the bias term to guarantee the convergence of 
$\mathrm{TV}(P_{f_0}, P_{\widehat{f}})$. For that purpose, we introduce the regularization 
to the estimator~\eqref{eqn:STV-learn}.

\subsection{Regularized STV Learning}
When we use the learning with the STV distance, the bias term appears in the upper bound of the
estimation error. In order to control the bias term, let us consider the learning with regularization, 
\begin{align}
 \label{eqn:reg-STV-learn}
 \min_{f\in\mathcal{H}_r}\mathrm{STV}_{\mathcal{H}_U,\sigma}(P_f,P_n)\ \ 
 \longrightarrow  \widehat{f}_r, 
\end{align}
where $f$ is rectricted to $\mathcal{H}_r$ with a positive constant $r$ that possibly depends on the sample size $n$. 
When the volume of $\mathcal{X}$, i.e., $\int_{\mathcal{X}}\mathrm{d}\mu$, is finite, the regularized $\mathrm{STV}$-based learning, 
\begin{align}
 \label{eqn:reg-STV-learn_additive}
 \min_{f\in\mathcal{H}} \mathrm{STV}_{\mathcal{H}_U,\sigma}(P_f,P_n)
 +\frac{1}{r^2} \|f\|^2
 \ \longrightarrow \  \widehat{f}_{\mathrm{reg},r}, 
\end{align}
leads to a solution that is similar to $\widehat{f}_r$. Indeed, under Assumption~(B) for $\sigma$, we have 
\begin{align*}
 \frac{1}{r^2} \|\widehat{f}_{\mathrm{reg},r}\|^2
& \leq 
 \mathrm{STV}_{\mathcal{H}_U,\sigma}(P_{\widehat{f}_{\mathrm{reg},r}},P_n)
 +\frac{1}{r^2} \|\widehat{f}_{\mathrm{reg},r}\|^2\\
& \leq 
 \mathrm{STV}_{\mathcal{H}_U,\sigma}(P_{0},P_n)
 +\frac{1}{r^2}\|0\|^2\leq 1, 
\end{align*}
where $P_0=p_0\dmu$ is the uniform distribution on $\mathcal{X}$ with respect to $\mu$. % defined from $P_f$ with $f=0$. 
Therefore, we have $\|\widehat{f}_{\mathrm{reg},r}\|\leq r$. 
The following theorem shows the estimation accuracy of regularized STV learning. 
\begin{theorem}
 \label{thm:reg-STV-learning}
 Assume Assumption~(C). Let $\sigma$ be the sigmoid function, which satisfies Assumption~(B). 
 Suppose that the dimension $d=\dim\mathcal{H}$ is finite. We assume that $\|f_0\|\leq r$ and $U\geq 2r$. 
 Then, the TV distance between the target distribution and the above regularized estimators 
 is bounded above by
\begin{align*}
 \mathrm{TV}(P_{f_0}, P_{\widehat{f}_r})
 &\lesssim
 \frac{r}{U}+
 \varepsilon + \sqrt{\frac{d}{n}} + \sqrt{\frac{\log(1/\delta)}{n}},\\
 \mathrm{TV}(P_{f_0}, P_{\widehat{f}_{\mathrm{reg},r}})
 &\lesssim
 \frac{r}{U}+ \varepsilon + \sqrt{\frac{d}{n}} + \frac{\|f_0\|^2}{r^2}+\sqrt{\frac{\log(1/\delta)}{n}}
\end{align*}
with probability greater than $1-\delta$. 
\end{theorem}
Let us consider the choice of $r$ and $U$. 
If $r/U$ is of the order $O(\sqrt{d/n})$, the convergence rate 
of $\mathrm{TV}(P_{f_0}, P_{\widehat{f}_r})$ is $O_p(\varepsilon+\sqrt{d/n})$. 
This is easily realized by setting $U=r\sqrt{n/d}$ and $r=r_n$ 
as any increasing sequence to infinity as $n\rightarrow\infty$. 
On the other hand, the order of $1/r^2$ appears in the upper bound of 
$\mathrm{TV}(P_{f_0}, P_{\widehat{f}_{\mathrm{reg},r}})$. 
By setting $r\geq n^{1/4}$ and $U\geq r\sqrt{n}$, 
we find that $\mathrm{TV}(P_{f_0}, P_{\widehat{f}_{\mathrm{reg},r}})  = O_p(\varepsilon+\sqrt{d/n})$. 
Therefore, with the appropriate setting of $U$ and $r$, 
both 
$\mathrm{TV}(P_{f_0}, P_{\widehat{f}_r})$ and 
$\mathrm{TV}(P_{f_0}, P_{\widehat{f}_{\mathrm{reg},r}})$ 
attain the order of $\varepsilon+\sqrt{d/n}$. % using $\mathcal{H}_U$ with the dimension $d$. 
\begin{corollary}
 \label{cor:reg-STV-learning}
 Assume the assumption of Theorem~\ref{thm:reg-STV-learning} and $d=\dim\mathrm{\mathcal{H}}$. 
 By setting $r\geq n^{1/4}$ and $U\geq r\sqrt{n}$, 
 the inequalities
 \begin{align*}
  \mathrm{TV}(P_{f_0}, P_{\widehat{f}_r}) \lesssim \varepsilon+\sqrt{\frac{d}{n}},\quad\text{and}\quad
  \mathrm{TV}(P_{f_0}, P_{\widehat{f}_{\mathrm{reg},r}}) \lesssim \varepsilon+\sqrt{\frac{d}{n}}
 \end{align*}
 hold with a high probability. 
\end{corollary}

Furthermore, we consider the STV learning in which the constraint $u\in\mathcal{H}_U$ in $\mathrm{STV}_{\mathcal{H}_U,\sigma}$
is replaced with the regularization term, 
\begin{align}
& \min_{f\in\mathcal{H}} \max_{u\in\mathcal{H},b\in\Rbb}
 \mathbb{E}_{P_f}[\sigma(u(X)-b)]- \mathbb{E}_{P_n}[\sigma(u(X)-b)]
 -\frac{1}{U^2} \|u\|^2 +\frac{1}{r^2} \|f\|^2\nonumber\\
& \longrightarrow \ \check{f}_{\mathrm{reg},r}. 
 \label{eqn:reg-STV-learn_additive-full}
\end{align}
Let $\check{u}\in\mathcal{H}$ and $\check{b}\in\Rbb$ be the optimal solution in the inner maximum problem for a fixed $f$. 
Then, we have 
\begin{align*}
 \frac{1}{U^2}\|\check{u}\|^2
 \leq 
 \mathbb{E}_{P_f} 
 [\sigma(\check{u}(X)-\check{b})] - \mathbb{E}_{P_n}[\sigma(\check{u}(X)-\check{b})]
 \leq 1. 
\end{align*}
The first inequality is obtained by comparing $u=\check{u}$ and $u=0$. 
In the same way as \eqref{eqn:reg-STV-learn_additive}, 
the norm of $\check{f}_{\mathrm{reg},r}$ is bounded above as follows. 
\begin{align*}
& \phantom{\leq}\frac{1}{r^2}\|\check{f}_{\mathrm{reg},r}\|^2\\
 &\leq 
 \frac{1}{r^2}\|\check{f}_{\mathrm{reg},r}\|^2
 +\max_{u,b}\mathbb{E}_{P_{\check{f}_{\mathrm{reg},r}}}[\sigma(u(X)-b)]- \mathbb{E}_{P_n}[\sigma(u(X)-b)]-\frac{1}{U^2}\|u\|^2\\
 &\leq 
 \max_{u,b}\mathbb{E}_{P_{0}}[\sigma(u(X)-b)]- \mathbb{E}_{P_n}[\sigma(u(X)-b)]-\frac{1}{U^2}\|u\|^2\qquad (\text{setting $f=0$})\\
 &\leq 1. 
\end{align*}
Thus, we have  $\|\check{f}_{\mathrm{reg},r}\|\leq r$ and $\|\check{u}\|\leq U$. 
\begin{theorem}
 \label{thm:fullreg-STV-learning}
 Assume Assumption~(C). Let $\sigma$ be the sigmoid function, which satisfies Assumption~(B). 
 Suppose that the dimension $d=\dim\mathcal{H}$ is finite. 
 We assume that $\|f_0\|\leq r$ and $U\geq 2r$. 
 The TV distance between the target distribution and the regularized estimator $\check{f}_{\mathrm{reg},r}$
 in \eqref{eqn:reg-STV-learn_additive-full} is bounded above as follows, 
\begin{align*}
 \mathrm{TV}(P_{f_0}, P_{\check{f}_{\mathrm{reg},r}})
 &\lesssim
 \left(\frac{r}{U}\right)^{2/3}
 +\frac{\|f_0\|^2}{r^2} +\varepsilon + \sqrt{\frac{d}{n}}+\sqrt{\frac{\log(1/\delta)}{n}}. 
\end{align*}
with probability greater than $1-\delta$. 
\end{theorem}
\begin{corollary}
 \label{cor:fullreg-STV-learning}
 Assume the assumption in Theorem~ \ref{thm:fullreg-STV-learning}. 
 The estimator $\check{f}_{\mathrm{reg},r}$ with $r\geq n^{1/4}$ and $U=rn^{3/4}$, %and $\widetilde{U}=r^{1/3} U^{2/3}$
 satisfies 
 \begin{align*}
  \mathrm{TV}(P_{f_0}, P_{\check{f}_{\mathrm{reg},r}}) \lesssim \varepsilon+\sqrt{\frac{d}{n}}
 \end{align*}
 with a high probability for a large $n$. 
\end{corollary}
The regularization parameters, $r=n^{1/4}$ and $U=n$, agree to the assumption in the above Corollary. 

The above convergence analysis shows that the regularization for the ``discriminator'' $u\in\mathcal{H}_U$
should be weaker than that for the ``generator'' $f\in\mathcal{H}_r$, i.e., 
$r<U$ is preferable to achieve a high prediction accuracy. 
This is because the bias term induced by STV distance is bounded above by $r/U$ up to a constant factor. 
The standard theoretical analysis of the GAN does not take the bias of the surrogate loss into account. 
That is, the loss function used in the probability density estimation is again used to evaluate the estimation error. 

In our setting, we assume that the expectation of $\sigma(u(X)-b)$ for $X\sim P_f$ is exactly computed. 
In Section~\ref{Approximation_and_Algorithm}, we consider the approximation of the expectation by the Monte Carlo method. 
We can evaluate the required sample size of the Monte Carlo method.

\subsection{Regularized STV learning for Infinite-dimensional Kernel Exponential Family}

In this section, we consider the regularized STV learning for infinite dimensional RKHS. 
In the definition of the STV distance, 
suppose that the range of the bias term $b$ is restricted to $\abs{b}\leq U$ when we deal with infinite-dimensional models. 
For the kernel exponential family $\mathcal{P}_{\mathcal{H}}$ defined from the infinite-dimensional RKHS $\mathcal{H}$, 
we employ the regularized STV learning, \eqref{eqn:reg-STV-learn},  \eqref{eqn:reg-STV-learn_additive}, and
 \eqref{eqn:reg-STV-learn_additive-full} 
to estimate the probability density $p_{f_0}$ using contaminated samples. 
\begin{theorem}
 \label{thm:infinite-STV-learning-bound}
 Assume Assumption~(C). Let $\sigma$ be the sigmoid function, which satisfies Assumption~(B). 
 Suppose $\sup_{x}k(x,x)\leq 1$. 
 Then, for $\|f_0\|\leq r$,
 the estimation error bounds of the regularized STV learning, 
 \eqref{eqn:reg-STV-learn}, \eqref{eqn:reg-STV-learn_additive}, and 
 \eqref{eqn:reg-STV-learn_additive-full}, are given by 
 \begin{align*}
 \mathrm{TV}(P_{f_0}, P_{\widehat{f}_r})
 & \lesssim
 \frac{r}{U} + \varepsilon +  \frac{U}{\sqrt{n}} + \sqrt{\frac{\log(1/\delta)}{n}}, \\
  %%% %%% %%% %%% %%% %%% %%% %%% %%% 
 \mathrm{TV}(P_{f_0}, P_{\widehat{f}_{\mathrm{reg},r}})
 & \lesssim
 \frac{r}{U} + \varepsilon +  \frac{U}{\sqrt{n}} + \sqrt{\frac{\log(1/\delta)}{n}}+ \frac{\|f_0\|^2}{r^2}, \\
  %%% %%% %%% %%% %%% %%% %%% %%% %%% 
  \mathrm{TV}(P_{f_0}, P_{\check{f}_{\mathrm{reg},r}})
 & \lesssim
  \left(\frac{r}{U}\right)^{2/3} + \varepsilon +  \frac{U}{\sqrt{n}} + \sqrt{\frac{\log(1/\delta)}{n}}+ \frac{\|f_0\|^2}{r^2}
 \end{align*}
 with probability greater than $1-\delta$. 
\end{theorem}
The proof is shown in Appendix~\ref{app:Proof-thm:infinite-STV-learning-bound}. 

The model complexity of the infinite-dimensional model is bounded above by $U/\sqrt{n}$ 
in the upper bound in Theorem~\ref{thm:infinite-STV-learning-bound}, 
while that of the $d$-dimensional model is bounded above by $\sqrt{d/n}$. 
Because of that, we can find the optimal order of $U$. 
\begin{corollary}
 \label{cor:infinite-dim-RKHS-optbound}
 Assume the assumption of Theorem~\ref{thm:infinite-STV-learning-bound}. Then, 
 \begin{align*}
 & \mathrm{TV}(P_{f_0}, P_{\widehat{f}_r}) \lesssim \varepsilon+\frac{1}{n^{1/4}},\quad
 \mathrm{TV}(P_{f_0}, P_{\widehat{f}_{\mathrm{reg},r}})\lesssim \varepsilon+\frac{1}{n^{1/5}},\\
 & \text{and}\ \ 
  \mathrm{TV}(P_{f_0}, P_{\check{f}_{\mathrm{reg},r}})\lesssim \varepsilon+\frac{1}{n^{1/6}}
\end{align*}
hold with a high probability, where the poly-log order is omitted in $\mathrm{TV}(P_{f_0}, P_{\widehat{f}_r})$. 
\end{corollary}

\begin{remark}
Using the localization technique introduced in Chapter 14 in \cite{wainwright19:_high}, one can obtain a detailed convergence rate for infinite dimensional exponential families. Suppose that the kernel function $k(x,x')$ is expanded as $k(x,x')=\sum_{i=1}^{\infty}\mu_i \phi_i(x)\phi_i(x')$, where $\phi_i,i=1,\ldots,$ are the orthonormal basis of the squared integrable function sets $L^2(P_{f_0})$. 
Suppose that $\mu_i$ is of the order $1/i^p,\,i=1,2,\ldots$. 
Using Theorem 14.20 in \cite{wainwright19:_high} with a minor modification, we find that
\begin{align*}
  \mathrm{TV}(P_{f_0}, P_{\widehat{f}_r})
 & \lesssim
 \varepsilon +  \frac{r}{U}+\delta_n U, \ \ 
    \delta_n = U^{\frac{1}{p+1}}n^{-\frac{p}{2(p+1)}}
\end{align*}
with probability greater than $1-\exp\{-cn\delta_n^2\}$, where $c$ is a positive constant. Setting $U=n^{\frac{p}{4p+6}}$ for a fixed $r$, we have $\mathrm{TV}(P_{f_0}, P_{\widehat{f}_r})\lesssim \varepsilon + n^{-\frac{p}{4p+6}}$ with probability greater than $1-\exp\{-cn^{\frac{3}{2p+3}}\}$. 
\end{remark}

\section{Accuracy of Parameter Estimation}
\label{sec:Accuracy_Par_est}

This section is devoted to studying the estimation accuracy in the parameter space. The proofs are deferred to Appendix~\ref{app:error_par_est}.

\subsection{Estimation Error in RKHS}
So far, we considered the estimation error in terms of the TV distance. 
In this section, we derive the estimation error in the RKHS. 
Such an evaluation corresponds to the estimation error of the finite-dimensional parameter in the Euclidean space. 
%This section provides not only the upper bound but a lower bound of the estimation error in RKHS. 
A lower bound of the TV distance is shown in the following lemma. 
%The proof is shown in Appendix~\ref{appendix:proof-lemma_TV_lowerbound}. 
\begin{lemma}
\label{lemma:lower-bd-TV_RKHS}
For the RKHS $\mathcal{H}$ with the kernel function $k:\mathcal{X}\times\mathcal{X}\rightarrow\mathbb{R}$, 
we assume $\sup_{x\in\mathcal{X}}k(x,x)\leq 1$. 
Suppose that $\int_{\mathcal{X}}\dmu=1$. 
Then, for $f,g\in\mathcal{H}_r$, it holds that
\begin{align*}
 \mathrm{TV}(P_f,P_g)
 \geq 
 \frac{\|f-g\|}{32r e^{2r}}\inf_{\substack{s\in\mathcal{H},\|s\|=1}}\int 
 \lvert s(x)-\mathbb{E}_{P_0}[s]\rvert^2\dmu, 
\end{align*}
where $P_0$ is the probability measure $p_f\dmu$ with $f=0$, i.e., the uniform distribution $p_0=1$ on $\mathcal{X}$
w.r.t. the measure $\mu$. 
\end{lemma}
For the RKHS $\mathcal{H}$, define
\begin{align*}
 \xi(\mathcal{H}):=\inf_{\substack{s\in\mathcal{H},\|s\|=1}}\int 
 \lvert s(x)-\mathbb{E}_{P_0}[s]\rvert^2\dmu.
\end{align*}
As shown in Lemma~\ref{lemma:lower-bd-TV_RKHS}, 
the RKHS norm is bounded above by the total variation distance multiplied by $\xi(\mathcal{H})$ for a fixed $r$. 

We consider the estimation error in the RKHS norm. 
Let $\mathcal{H}_1,\mathcal{H}_2,\ldots,\mathcal{H}_d,\dots$ be a sequence of RKHSs such that $\dim\mathcal{H}_d=d$. 
Our goal is to prove that 
the estimation error of the estimator $\widehat{f}_r$ in \eqref{eqn:reg-STV-learn} for the statistical model 
$\mathcal{P}_{\mathcal{H}_{d,r}}$ is of the order $\varepsilon+\sqrt{d/n}$, where $\mathcal{H}_{d,r}=(\mathcal{H}_{d})_r=\{f\in\mathcal{H}_d\,:\,\|f\|\leq r\}$. 
\begin{corollary}
 \label{thm:finite-dim-RKHS-minimax-opt}
 For a positive constant $r$, let us consider the finite-dimensional statistical model 
 $\mathcal{P}_{\mathcal{H}_{d,r}},\, d=1,2,\ldots$ such that $\inf_{d\in\mathbb{N}}\xi(\mathcal{H}_d)>0$. 
 Assume Assumption~(C) with $\mathcal{H}=\mathcal{H}_{d,r}$
 and Assumption~(B) with the sigmoid function $\sigma$. 
 Suppose that $f_0\in\mathcal{H}_{d,r}$. 
 Then, the estimation error of the estimator $\widehat{f}_r$ in \eqref{eqn:reg-STV-learn} with $U\gtrsim \sqrt{n}$ is 
\begin{align*}
 \|\widehat{f}_r-f_0\|\lesssim \varepsilon+\sqrt{\frac{d}{n}}
\end{align*}
 with a high probability. 
\end{corollary}
The proof is shown in Appendix~\ref{appendix:subsec_minimax-opt-RKHS}. 

When $\xi(\mathcal{H}_d)$ is not bounded below by a positive constant, the estimation accuracy is 
of the order $\varepsilon+\frac{1}{\xi(\mathcal{H}_d)}\sqrt{\frac{d}{n}}$, meaning that 
the dependency on the dimension $d$ is $\sqrt{d}/\xi(\mathcal{H}_d)$ that is greater than $\sqrt{d}$. 

If an upper bound of $\|f_0\|$ is unknown, $r$ is treated as the regularization parameter depending on $n$. 
In such a case, the coefficient of $\varepsilon$ depends on $n$, and it goes to infinity as $n\rightarrow\infty$. 
In theoretical analysis, the upper bound of $\|f_0\|$ is often assumed to be known especially 
for covariance matrix estimation~\citep{chen16:_huber,chen18:_robus,DBLP:journals/jmlr/GaoYZ20,liu21:_robus_w_gan_based_estim}. 

Let us construct an example of the sequence $\mathcal{H}_1,\mathcal{H}_2,\ldots$ satisfying the assumption of Corollary~\ref{thm:finite-dim-RKHS-minimax-opt}. 
Let $\{t_j(x)\}_{j=1}^{\infty}$ be the orthonormal functions which are orthogonal to the constant function under $\mu$, i.e, 
\begin{align*}
 \mathbb{E}_{P_0}[t_i(x)]=0, \ \ 
 \mathbb{E}_{P_0}[t_i(x)t_j(x)]=
 \begin{cases}
  1,\quad i=j,\\
  0,\quad i\neq{j}. 
 \end{cases}
\end{align*}
for $i,j=1,2,\ldots$. 
For a positive decreasing sequence $\{\lambda_k\}_{k=1}^{\infty}$ such that 
$\lambda_\infty:=\lim_{k\rightarrow\infty}\lambda_k>0$, 
define the sequence of finite-dimensional RKHSs, $\mathcal{H}_1, \mathcal{H}_2,\ldots,\mathcal{H}_d,\cdots$ by 
\begin{align}
 \label{eqn:finite-dim-RKHS}
 \mathcal{H}_d=\bigg\{ f(x)=\sum_{j=1}^{d} \alpha_j t_j(x)\,:\,\alpha_1,\ldots,\alpha_d\in\Rbb\bigg\}, 
\end{align}
where the inner product $\<f,g\>$ for $f=\sum_j\alpha_j t_j$ and $g=\sum_j\beta_j t_j$ is defined by 
$\<f,g\>=\sum_{i}\alpha_i\beta_i/\lambda_i$. 
Here, $t_1,\ldots,t_d$ are the sufficient statistic of $\mathcal{P}_{\mathcal{H}_d}$. 
Then, 
\begin{align*}
\xi(\mathcal{H}_d)
 =\inf_{\sum_{i=1}^{d}\alpha_i^2/\lambda_i=1}\sum_{i=1}^{d}\alpha_i^2
 \geq\inf_{\sum_{i=1}^{d}\alpha_i^2/\lambda_i=1}\sum_{i=1}^{d}\alpha_i^2\frac{\lambda_\infty}{\lambda_i}=\lambda_\infty> 0. 
\end{align*}

\begin{example}
 The construction of the RKHS sequence allows a distinct sample space for each $d$. 
 Suppose $\mu_d$ be the $d$-dimensional standard normal distribution on $\mathbb{R}^d$. 
 For $\x=(x_1,\ldots,x_d)$, define $t_i(\x)=x_i$. Let $\lambda_i=1$ for all $i$. 
 Then, the RKHS $\mathcal{H}_d$ is the linear kernel $k(\x,\z)=\sum_{i=1}^{d}\mu_i t_i(\x)t_i(\z)=\x^T\z$ 
 for $\x,\z\in\mathbb{R}^d$ and the corresponding statistical model is the multivariate normal distribution 
 $N_d(\f,I_d)$ with the parameter $\f\in\Rbb^d$. 
\end{example}

\begin{remark}
 For the infinite dimensional RKHS $\mathcal{H}$, 
 typically $\xi(\mathcal{H})=0$ holds. For example, the kernel function is defined as 
 $k(x,x')=\sum_{i=1}^{\infty}\lambda_i e_i(x)e_i(x')$ in the same way as above, where
 $\{\lambda_i\}$ is  a positive decreasing sequence such that $\lambda_i\rightarrow0$ as $i\rightarrow\infty$. 
 Hence, we have 
 $\xi(\mathcal{H})\leq \int\lvert\sqrt{\lambda_i}e_i(x)\rvert^2\dmu=\lambda_i\rightarrow0$ for $i\rightarrow\infty$. 
 Lemma~\ref{lemma:lower-bd-TV_RKHS} does not provide a meaningful upper bound of the estimation error 
 in such an infinite-dimensional RKHS. 
When the kernel function is bounded as assumed in Lemma~\ref{lemma:lower-bd-TV_RKHS}, 
the RKHS norm is an upper bound of the infinity norm, i.e., 
$\sup_{x}{\mid{f(x)-g(x)}\mid}\leq \|f-g\|$. 
On the other hand, TV distance is defined through the integral. Hence, even when the infinity
norm between two probability densities is large, the corresponding TV distance can be small. 
Such a situation can occur when $\mathcal{H}$ has a rich expressive power. 
We need another approach to derive a lower bound of the TV distance in terms of parameters in the RKHS.  
\end{remark}

\subsection{Robust Estimation of Normal Distribution}

Let us consider the robust estimation of parameters for the normal distribution. 

First of all, we prove that regularized STV learning provides a robust estimator of 
the mean vector in the multi-dimensional normal distribution 
with the identity covariance matrix. 
For $\mathcal{X}=\Rbb^d$, let us define $\rmd\mu(\x)=C_d e^{-\|\x\|_2^2/2}\rmd{\x}$ for $\x\in\Rbb^d$, 
where $C_d$ is the normalizing constant depending on the dimension $d$ 
such that $\int_{\Rbb^d}\rmd\mu(\x)=1$. 
Let $\mathcal{H}$ be the RKHS with the kernel function $k(\x,\z)=\x^T\z$ for $\x,\z\in\Rbb^d$. 
For the functions $f(\x)=\x^T\f$ and $g(\x)=\x^T\g$, the inner product  in the RKHS is $\<f, g\>=\f^T\g$. 
The probability density of the $d$-dimensional multivariate normal distribution with the identity covariance matrix, $N_d(\f,I)$, with respect to the measure $\mu$ is given by $p_f(\x)=\exp\{\x^T\f-\|\f\|_2^2/2\}$. 
Corollary~\ref{cor:reg-STV-learning} guarantees that the estimated mean vector $\widehat{\f}_{\mathrm{reg},r}$ by the regularized STV-learning satisfies
\begin{align*}
 \mathrm{TV}(N_d(\f_0,I), N_d(\widehat{\f}_{\mathrm{reg},r},I)) \lesssim \varepsilon + \sqrt{\frac{d}{n}}. 
\end{align*}
A lower bound of the TV distance between the $d$-dimensional normal distributions, $N_d(\bmu_1,I)$ and $N_d(\bmu_2,I)$, is presented in 
\cite{devroye18:_gauss}, 
\begin{align*}
 \frac{\min\{1,\|\bmu_1-\bmu_2\|_2\}}{200}\leq \mathrm{TV}(N_d(\bmu_1,I), N_d(\bmu_2,I)). 
\end{align*}
Therefore, we have 
\begin{align}
 \label{eqn:convergence-normal-mean-vector}
 \|\f_0-\widehat{\f}_{\mathrm{reg},r}\|_2\lesssim \varepsilon + \sqrt{\frac{d}{n}}
\end{align}
for a small $\varepsilon$ and a large $n$. 
Though Corollary~\ref{thm:finite-dim-RKHS-minimax-opt} provides the same conclusion, 
the above argument does not need the boundedness of $\|\f_0\|_2$. 
As shown in \cite{chen18:_robus}, a lower bound of the estimation error under the contaminated distribution is 
\begin{align*}
 \inf_{\widehat{\bm f}}
 \sup_{
 Q:\inf_{\f}\mathrm{TV}(N({\bm f},I),Q)\leq\varepsilon
 }
 Q \bigg(\|\widehat{\bm f}-{\bm f}\|_2\gtrsim \varepsilon + \sqrt{\frac{d}{n}}\bigg)>0. 
\end{align*}
Therefore, the estimator $\widehat{\f}_{\mathrm{reg},r}$ attains the minimax optimal rate. 

\begin{table}\centering 
\caption{Robust estimators of the covariance matrix for the multivariate normal distribution are compared: 
the matrix depth\citep{chen18:_robus}, $f$-GAN\citep{DBLP:journals/jmlr/GaoYZ20}, 
W-GAN\citep{liu21:_robus_w_gan_based_estim}, STV(our method). 
In the contamination column, the measure to quantify the contamination and the condition of $\varepsilon$ are shown, 
where 
``Hub.'', ``TV'', and ``Wass.'' respectively denote 
 the Huber contamination, TV distance, and Wasserstein distance. 
As shown in ``Eval.'' the estimation accuracy is evaluated by the operator norm except our analysis. 
The ``opt.'' column shows whether the estimator attains the minimax optimality. 
All the estimators are minimax optimal for the contaminated distribution with the conditions in ``Contamination.''}
 \label{tbl:CovarianceMatrix_Estimate}
\begin{tabular}{lllcc}
 Estimator  & \multicolumn{1}{c}{Model: $N_d(\0,\Sigma)$}  & 
	 \multicolumn{1}{c}{Contamination}  & Eval. & opt.  \\ \hline
Mat. depth%\cite{chen18:_robus} 
   &	$\|\Sigma\|_{\mathrm{op}}\leq M$ & Hub., $\varepsilon<1/5,\,\sqrt{\frac{d}{n}}\lesssim 1$        
	 &  $\|\widehat{\Sigma}-\Sigma\|_{\mathrm{op}}\lesssim \varepsilon+\sqrt{\frac{d}{n}}$ & $\checkmark$ \\
$f$-GAN  %\cite{DBLP:journals/jmlr/GaoYZ20}
 & $\|\Sigma\|_{\mathrm{op}}\leq M$ & 
	 TV, $\varepsilon+\sqrt{\frac{d}{n}}\lesssim 1$&  $\|\widehat{\Sigma}-\Sigma\|_{\mathrm{op}}\lesssim\varepsilon+\sqrt{\frac{d}{n}}$ & $\checkmark$\\
W-GAN %\cite{liu21:_robus_w_gan_based_estim} 
&
     $ \|\Sigma\|_{\mathrm{op}}^{\pm1}\leq M$ &Wass., $\varepsilon+\sqrt{\frac{d}{n}}\lesssim 1$ & 
	     $\|\widehat{\Sigma}-\Sigma\|_{\mathrm{op}}\lesssim\varepsilon+\sqrt{\frac{d}{n}}$ & $\checkmark$ \\\hline
STV
  & $\|\Sigma\|_{\mathrm{op}}\leq M$ & TV, $\varepsilon+\sqrt{\frac{d^2}{n}}\lesssim 1$&
$\|\widehat{\Sigma}-\Sigma\|_{\mathrm{F}}\lesssim\varepsilon+\sqrt{\frac{d^2}{n}}$ & $\checkmark$ \\\hline
\end{tabular}
\end{table}

Next, let us consider the estimator of the covariance matrix for the multivariate normal distribution with mean zero, $N_d(\0,\Sigma)$. 
For $\mathcal{X}=\Rbb^d$, 
let us define $\rmd\mu(\x)=C_d e^{-\|\x\|_2^2/2}\rmd{\x}$ for $\x\in\Rbb^d$ as above. 
Let $\mathcal{H}$ be the RKHS with the kernel function $k(\x,\z)=(\x^T\z)^2$ for $\x,\z\in\Rbb^d$. 
Let $f(\x)=\x^T F \x$ and $g(\x)=\x^T G\x$ be 
quadratic functions defined by symmetric matrices $F$ and $G$. 
Their inner product in the RKHS is given by $\<f,g\>=\mathrm{Tr}F^TG$. 
The probability density w.r.t. $\mu$ is defined by $p_f(\x)=\exp\{-\x^T F \x/2-A(F)\}$ for $f(\x)=\x^T F\x$,
which leads to the statistical model of the normal distribution $N_d(\0,(I+F)^{-1})$. 
Let $P_{f_0}$ be the normal distribution $N_d(\0, \Sigma_0)$ and 
$P_{\widehat{f}}$ be the estimated distribution $N_d(\0, \widehat{\Sigma})$ using the regularized STV learning. 
Here, the estimated function $\widehat{f}(\x)=\x^T \widehat{F}\x$ in the RKHS 
provides the estimator of the covariance matrix $\widehat{\Sigma} = (I+\widehat{F})^{-1}$. 
Then, 
\begin{align*}
\mathrm{TV}(N_d(\0,\Sigma_0),\,N_d(\0,\widehat{\Sigma}))\lesssim \varepsilon + \sqrt{\frac{d^2}{n}}
\end{align*}
holds with high probability. 
As shown in \cite{devroye18:_gauss}, 
the lower bound of the TV distance between $N_d(\0,\Sigma_0)$ and $N_d(\0,\Sigma_1)$ is given by 
\begin{align*}
 \frac{\min\{1,\|\Sigma_0\|_{\mathrm{op}}^{-1}\|\Sigma_1-\Sigma_0\|_{\mathrm{F}}\}}{100}\leq \mathrm{TV}(N_d(\0,\Sigma_0), N_d(\0,\Sigma_1)), 
\end{align*}
where $\|\cdot\|_{\mathrm{op}}$ is the operator norm defined as the maximum singular value and 
$\|\cdot\|_{\mathrm{F}}$ is the Frobenius norm. 
If $\varepsilon+\sqrt{d^2/n}\lesssim 1$, we have 
\begin{align*}
 \|\Sigma_0-\widehat{\Sigma}\|_{\mathrm{F}} \lesssim \varepsilon + \sqrt{\frac{d^2}{n}}, 
\end{align*}
where $\|\Sigma_0\|_{\mathrm{op}}$ is regarded as a constant. 
 \cite{DBLP:journals/jmlr/GaoYZ20}
 proved that the estimater $\widehat{\Sigma}$ based on the GAN method attains the error bound in terms of the operator norm, 
 \begin{align*}
  \|\Sigma_0-\widehat{\Sigma}\|_{\mathrm{op}} \lesssim \varepsilon + \sqrt{\frac{d}{n}}. 
 \end{align*}
 The relationship between the Frobenius norm and the operator norm leads to the inequality, 
 \begin{align*}
  \|\Sigma_0-\widehat{\Sigma}\|_{\mathrm{F}}\leq 
  \sqrt{d}\|\Sigma_0-\widehat{\Sigma}\|_{\mathrm{op}}
  \lesssim \sqrt{d} \varepsilon + \sqrt{\frac{d^2}{n}}. 
 \end{align*}
A naive application of the result in \cite{DBLP:journals/jmlr/GaoYZ20}
leads to $\sqrt{d}$ factor to $O(\varepsilon)$ term. 

In the same way as the estimation of the mean vector, the estimator $\widehat{\Sigma}$
attains the minimax optimality. Indeed, the following theorem holds. 
The proof is shown in Appendix~\ref{appendix:subsec_minimax-opt-NormalCov}. 
\begin{theorem}
 \label{thm:minimax-opt-NormalCov}
Let us consider the statistical model $\mathcal{N}_R=\{N_d(\0, \Sigma)\,:\,\Sigma\in \mathfrak{T}_R\}$, 
where $\mathfrak{T}_R=\{\Sigma\in\Rbb^{d\times d}\,:\,\Sigma\succ{O},\, \|\Sigma\|_{\mathrm{op}}\leq 1+R\}$
for a positive constant $R>1/2$. Then, 
\begin{align*}
 \inf_{\widehat{\Sigma}}\sup_{Q:{\displaystyle \inf_{\Sigma\in\mathfrak{T}_R}}\!\!
 \mathrm{TV}(N_d(\0,\Sigma),Q)<\varepsilon}
 Q\bigg(\|\widehat{\Sigma}-\Sigma\|_{\mathrm{F}} \gtrsim\varepsilon+\sqrt{\frac{d^2}{n}}\bigg)>0  
\end{align*}
holds for $\varepsilon\leq 1/2$. The above lower bound is valid even for the unconstraint model, i.e.,
 $R=\infty$. 
\end{theorem}
The covariance estimator based on STV learning attains the minimax optimal rate for the Frobenius norm, 
while the optimality in the operator norm is studied 
in several works~\citep{chen18:_robus,DBLP:journals/jmlr/GaoYZ20,liu21:_robus_w_gan_based_estim}. 
Table~\ref{tbl:CovarianceMatrix_Estimate} shows theoretical results revealed by some works.

Let us show simple numerical results to confirm the feasibility of our method. 
We compare robust estimators for the mean vector and the full covariance
matrix under Huber contamination. 
The expectation in the loss function is computed by sampling from the normal distribution. 

For the mean vector of the multivariate normal distribution, 
the model $p_f(\x)=\exp\{\x^T\f-\|\f\|_2^2/2\}$ and the RKHS $\mathcal{H}_U$ endowed with the linear kernel are used. 
The data is generated from $0.9 N_d(\0, I)+0.1 N_d(5\cdot{\bm e},I)$ and the target is to estimate the mean vector of $N_d(\0, I)$, where ${\bm e}$ is the $d$-dimensional vector $(1,1,\ldots,1)^T\in\Rbb^d$. 
We examine the estimation accuracy of the component-wise median, the GAN-based estimator using Jensen-Shannon loss, and the minimum TV distance estimator.  
We used the regularized STV-based estimator \eqref{eqn:reg-STV-learn_additive-full} with the gradient descent/ascent method to solve the min-max optimization problem. 
The regularization parameters are set to $1/U^2=10^{-4}$ and $1/r^2=3\times 10^{-5}$. 
The results are presented in Tables \ref{tbl:mean-data} and \ref{tbl:mean-dim}. 
As shown in theoretical analysis in \cite{chen18:_robus}, 
the component-wise median is sub-optimal. 
The TV-based estimator is not efficient, partially because of the difficulty of solving the optimization problem. 
As the smoothed variant of the TV-based estimator, GAN-based method and our method provide reasonable results. 

\begin{table}\centering 
\caption{Estimation of the mean vector for $N_d(\f,I)$ with $d=50$ and   the contamination ratio $\varepsilon=0.1$. The sample size varies from 50 to 1000. The averaged estimation error in the $d$-dimensional Euclidean norm is reported with the standard deviation. }
\label{tbl:mean-data}
\begin{tabular}{r|cccc}
$n$   &\multicolumn{1}{c}{JS-GAN} &\multicolumn{1}{c}{TV}&\multicolumn{1}{c}{ Comp. Median} &\multicolumn{1}{c}{reg. STV}\\\hline
50	 &2.639 (1.388)&2.105 (0.376)&1.734 (0.042)&1.391 (0.237) \\
100	 &1.397 (0.537)&1.521 (0.222)&1.392 (0.069)&0.961 (0.124) \\
200	 &0.963 (0.138)&1.191 (0.227)&1.228 (0.092)&0.718 (0.087) \\
500	 &0.536 (0.081)&1.096 (0.139)&1.111 (0.125)&0.442 (0.051) \\
1000	 &0.326 (0.034)&1.070 (0.264)&1.062 (0.105)&0.373 (0.056) \\
\end{tabular}
\end{table}

\begin{table}\centering 
\caption{Estimation of the mean vector for $N_d(\f,I)$ with the sample size 
$n=1000$ and the contamination ratio $\varepsilon=0.1$. The dimension of data varies from 10 to 100. 
The averaged estimation error in the $d$-dimensional Euclidean norm is reported with the standard deviation. }
\label{tbl:mean-dim}
\begin{tabular}{r|cccc}
 $d$ &\multicolumn{1}{c}{JS-GAN} &\multicolumn{1}{c}{TV}&\multicolumn{1}{c}{ Comp. Median} &\multicolumn{1}{c}{reg. STV}\\\hline
10	&0.094 (0.020)&0.371 (0.210) &0.442 (0.043)&0.178 (0.057)\\
25	&0.206 (0.022)&0.743 (0.331) &0.730 (0.069)&0.299 (0.045)\\
50	&0.332 (0.046)&1.153 (0.349) &1.022 (0.092)&0.399 (0.070)\\
75	&0.405 (0.069)&1.185 (0.374) &1.256 (0.125)&0.482 (0.105)\\
100	&0.555 (0.052)&1.326 (0.172) &1.505 (0.105)&0.712 (0.333)\\
\end{tabular}
\end{table}

For the estimation of the full covariance matrix 
for the multivariate normal distribution, 
the model $p_f(\x)=\exp\{\x^T F \x-A(F)\}$ and the RKHS $\mathcal{H}_U$ with the quadratic kernel are used. 
The data is generated from $0.8 N_d(\0,\Sigma)+0.2 N(6\cdot {\bm e}, \Sigma)$, where $\Sigma_{ij}=2^{-\abs{i-j}}$. 
We examined the STV-based estimator with Kendall's rank correlation coefficient~\cite{kendall1938measure} and GAN-based method~\cite{DBLP:journals/jmlr/GaoYZ20}. 
We used the regularized STV-based estimator  \eqref{eqn:reg-STV-learn_additive-full} with the gradient descent/ascent method to solve the min-max optimization problem. 
The regularization parameters are set to $1/U^2=10^{-4}$ and $1/r^2=10^{-4}$. 
The results are presented in Table~\ref{tbl:variance-estimation}. 
Overall, the GAN-based estimator outperforms the other method. 
This is because the optimization technique of the GAN-based estimator using deep neural networks is 
highly developed in comparison to the STV-based estimator. 
In the STV-based method, the optimization algorithm sometimes encounters a sub-optimal solution, though the objective function is smooth. 
Developing an efficient algorithm for the STV-based method is important future work. 

\begin{table}\centering 
\caption{Median of the Frobenius norm with median absolute deviation. 
Estimation of the covariance matrix for $N_d(\0,\Sigma)$ with the sample size $n=50000$ and the contamination ratio $\varepsilon=0.2$. 
The dimension of data varies from 5 to 25. 
The median of the estimation error in the Frobenius norm is reported with the median absolute deviation. }
\label{tbl:variance-estimation}
\begin{tabular}{r|lrl}
$d$  &\multicolumn{1}{c}{JS-GAN}  &\multicolumn{1}{c}{Kendall's tau}&\multicolumn{1}{c}{reg. STV} \\\hline
5      &0.083 (0.023) &  4.761 (0.088) & 0.912 (0.062)\\
10     &0.130 (0.027) &  9.292 (0.249) & 1.657 (0.471)\\
25     &0.203 (0.007) & 23.083 (0.459) & 2.051 (0.480)\\
\end{tabular}
\end{table}

\section{Approximation of Expectation and Learning Algorithm}
\label{Approximation_and_Algorithm}

Let us consider how to solve the min-max problem of regularized STV learning. 
The optimization problem is given by 
\begin{align*}
 \sup_{f\in\mathcal{H}_r} \inf_{u\in\mathcal{H}_U,\abs{b}\leq U}
 \mathbb{E}_{P_f}[\sigma(u(X)-b)] - \mathbb{E}_{P_n}[\sigma(u(X)-b)], 
\end{align*}
where $\sigma$ is the sigmoid function. For the optimization problems on the infinite-dimensional RKHS, usually 
the representer theorem applies to reduce the infinite-dimensional optimization problem to the finite-dimensional one. 
However, the representer theorem does not work to the above optimization problem, 
since the expectation w.r.t. $P_f, f\in\mathcal{H}_r$ appears. 
That is, the loss function may depend on $f(x)$ of all $x\in\mathcal{X}$. 
Even for the finite-dimensional RKHS, the computation of the expectation is often intractable. 
Here, we use the importance sampling method to overcome the difficulty. 

We approximate the expectation w.r.t. $P_f$ by the 
sample mean over $Z_1,\ldots,Z_\ell\sim q$, where $q$ is an arbitrary density function such that 
the sampling from $q$ and the computation of the probability density $q(z)$ is feasible. 
A simple example is $q=p_0$, i.e., the uniform distribution on $\mathcal{X}$ w.r.t. the base measure~$\mu$. 
Let us define $\sigma_X:=\sigma(u(X)-b)$. The expectation $\mathbb{E}_{X\sim P_{f}}[\sigma(u(X)-b)]$ is approximated by 
$\bar{\sigma}_\ell$, 
\begin{align}
\label{eqn:bar-sigma}
\bar{\sigma}_\ell:=
 \frac{\frac{1}{\ell}\sum_{j=1}^{\ell} \frac{e^{f(Z_j)}}{q(Z_j)}\sigma_{Z_j}}
 {\frac{1}{\ell}\sum_{j=1}^{\ell} \frac{e^{f(Z_j)}}{q(Z_j)}}
 \approx
 \mathbb{E}_{Z\sim q}\bigg[\frac{p_f(Z)}{q(Z)}\sigma_Z\bigg]
 =
\mathbb{E}_{X\sim P_{f}}[\sigma(u(X)-b)]. 
\end{align}
As an approximation of $X\sim P_f$, we employ the probability distribution on $Z_1,\ldots,Z_\ell$ defined by
\begin{align}
 \label{eqn:approximate-dist-hatP}
\widehat{P}_f(Z=Z_i)=\frac{e^{f(Z_i)}/q(Z_i)}{\sum_{j=1}^{\ell}e^{f(Z_j)}/q(Z_j)},\ \ Z_i\sim_{i.i.d.} q,  \quad i=1,\ldots,\ell. 
\end{align}
Then, an approximation of the estimator $\widehat{f}_r$ is given by the minimizer 
of $\mathrm{STV}_{\mathcal{H}_U,\sigma}(P_n,\widehat{P}_f)$, i.e., 
\begin{align*}
 \min_{f\in\mathcal{H}_r}\mathrm{STV}_{\mathcal{H}_U,\sigma}(P_n,\widehat{P}_f)\  \longrightarrow \ \widetilde{f}_r. 
\end{align*}
The approximation, $\widetilde{f}_r$, is usable regardless of the dimensionality of the model $\mathcal{H}$. 
Let us evaluate the error bound of the approximate estimator $P_{\widetilde{f}_r}$. 
\begin{theorem}
 \label{theorem:error_rate-approximate-STV-learning}
 Assume Assumption~(C). 
 Suppose that $\sup_{x\in\mathcal{X}}k(x,x)\leq K^2$. 
 Let us consider the approximate estimator $P_{\widetilde{f}_r}$ obtained by the regularized STV learning using 
 $\mathrm{STV}_{\mathcal{H}_U,\sigma}$, where $\sigma$ is the sigmoid function satisfying Assumption~(B)
 and the constraint $\abs{b}\leq U$ is imposed the STV distance. 
 Then, the estimation error of $P_{\widetilde{f}_r}$ is 
\begin{align*}
  \mathrm{TV}(P_{f_0}, P_{\widetilde{f}_r})
\lesssim
\varepsilon + 
 \frac{r}{U} + 
 \frac{C_{{\mathcal{H}_U}}}{\sqrt{n}}
  +
 \frac{e^{Kr}(r+U)}{\sqrt{\ell}} + 
 \sqrt{\log\frac{1}{\delta}}\left(\frac{1}{\sqrt{n}}+\frac{e^{Kr}}{\sqrt{\ell}}\right)
\end{align*}
with probability greater than $1-\delta$, where 
$\mathrm{C}_{\mathcal{H}_U}=UK$ for the infinite-dimensional $\mathcal{H}_U$
and 
$\mathrm{C}_{\mathcal{H}_U}=\sqrt{d}$ for the finite-dimensional $\mathcal{H}_U$ with $d=\dim\mathcal{H}$. 
\end{theorem}

For the finite-dimensional RKHS $\mathcal{H}$ with the dimension $d$ endowed with the bounded kernel, 
the convergence rate of $P_{\widetilde{f}}$ with $U=r\sqrt{n}$ and $\ell=n^2 r^2 e^{2Kr}$ attains 
\begin{align*}
\mathrm{TV}(P_{f_0}, P_{\widetilde{f}_r})\lesssim \varepsilon+\sqrt{\frac{d}{n}}. 
\end{align*}
We expect that a tighter bound for the approximation will lead to a more reasonable sample size such as $\ell=O(n)$. 

For the infinite-dimensional RKHS, 
the convergence rate of $P_{\widetilde{f}}$ with $\ell=n,\, U=n^{1/4},\, r=O(\log\log{n})$ attains the following bound
\begin{align*}
 \mathrm{TV}(P_{f_0}, P_{\widetilde{f}_r})\lesssim \varepsilon+\frac{1}{n^{1/4}}
\end{align*}
with high probability, where the poly-log order is omitted. %.  $\widetilde{O}(\cdot)$ includes the 
Since the loss function of the approximate estimator depends on $f$ via $f(Z_j),j=1,\ldots,\ell$, 
the representation theorem works to compute the estimator~$\widetilde{f}_r$. 
Indeed, 
the optimal solution of $\widetilde{f}_r$ is expressed by the linear sum of $k(Z_j,\cdot),j=1,\ldots,\ell$ and the optimal solution of $u$ is expressed by the linear sum of $k(Z_j,\cdot),j=1,\ldots,\ell$ and $k(X_i,\cdot),i=1,\ldots,n$. 
Hence, the problem is reduced to solve the finite-dimensional problem. 

According to Theorem~\ref{theorem:error_rate-approximate-STV-learning}, the exponential term, $e^{Kr}$, appears in the upper bound, meaning that we need a strong regularization to surpress that term. Recently, \cite{dai2019kernel,dai2020exponential} proposed an approximation method of the expectation in the penalized MLE for the kernel exponential family. It is an important issue to verify whether similar methods are also effective for STV-learning.

\section{Concluding Remarks}
%\section{Discussion and Future Works}
\label{Conclusion}

Since \cite{gao19:_robus_estim_via_gener_adver_networ} connected the classical robust estimation with GANs, 
computationally efficient robust estimators using deep neural networks have been proposed. 
Most existing works focus on the theoretical analysis of estimators for the normal mean and covariance matrix
under various conditions for contaminated distributions. 

In this paper, we studied the IPM-based robust estimator for the kernel exponential family, including 
infinite-dimensional models for probability densities. As a class of IPMs, we defined the STV distance. 
The relationship between the TV distance and STV distance has an important role to evaluate the 
estimation accuracy of the STV-based robust estimator. 
For the covariance matrix estimation of the multivariate normal distribution, 
we proved the STV-based estimator attains the minimax optimal estimator for the Frobenius norm, 
while existing estimators are minimax optimal for the operator norm. 
Furthermore, 
we proposed an approximate STV-based estimator using importance sampling 
to mitigate the computational difficulty. 
The framework studied in this paper is regarded as a natural extension of existing IPM-based robust estimators for multivariate normal distribution to more general statistical models with theoretical guarantees. 

The computational efficiency and stability of STV learning have room for improvement. 
Recent progress in robust statistics research has also been made from the viewpoint of computational algorithms; 
see \cite{diakonikolas16:_robus_estim_high_dimen_comput_intrac,diakonikolas2023algorithmic}. 
Incorporating these results into our research to improve our algorithms is an important research direction for practical applications.

\bmhead{Acknowledgments}
This work was partially supported by JSPS KAKENHI Grant Numbers 19H04071, 20H00576, and 23H03460.

\section*{Declarations}
\begin{description}
    \item[Data availability] The dataset analyzed during the current study is available in the GitHub, https://github.co.jp/
    \item[Conflict of interest]
     The corresponding author states that there is no conflict of interest.
\end{description}

Some journals require declarations to be submitted in a standardised format. Please check the Instructions for Authors of the journal to which you are submitting to see if you need to complete this section. If yes, your manuscript must contain the following sections under the heading `Declarations':

\begin{itemize}
\item Funding
\item Conflict of interest/Competing interests (check journal-specific guidelines for which heading to use)
\item Ethics approval 
\item Consent to participate
\item Consent for publication
\item Availability of data and materials
\item Code availability 
\item Authors' contributions
\end{itemize}

\noindent
If any of the sections are not relevant to your manuscript, please include the heading and write `Not applicable' for that section. 

\bigskip
\begin{flushleft}%
Editorial Policies for:

\bigskip\noindent
Springer journals and proceedings: \url{https://www.springer.com/gp/editorial-policies}

\bigskip\noindent
Nature Portfolio journals: \url{https://www.nature.com/nature-research/editorial-policies}

\bigskip\noindent
\textit{Scientific Reports}: \url{https://www.nature.com/srep/journal-policies/editorial-policies}

\bigskip\noindent
BMC journals: \url{https://www.biomedcentral.com/getpublished/editorial-policies}
\end{flushleft}

\begin{appendices}

\section{Proofs of Proposition and Lemmas in Section~\ref{sec:STV}}
\label{app:proof-modulus}

\subsection{Proof of Lemma~\ref{lemma-sigma0}}
\begin{proof}
 For any fixed $u\in\mathcal{F}$, let us define $\xi(b)$ and $\bar{\xi}(b)$ by 
\begin{align*}
\xi(b)= Q\big(u(x)<b \big)-P\big(u(x)<b \big),\quad 
\bar{\xi}(b)= Q\big(u(x)\leq b \big)-P\big(u(x)\leq b \big). 
\end{align*}
Note that $\lim_{b\rightarrow\pm\infty}\xi(b)=\lim_{b\rightarrow\pm\infty}\bar{\xi}(b)=0$. 
We see that $\xi(b)$ is left-continuous and $\bar{\xi}(b)$ is right-continuous. 
From the definition of the distribution function, $\xi(b)=\bar{\xi}(b)$ holds except for discontinuous points. 
Furthermore, for any $b_0\in\Rbb$, 
$\lim_{b\rightarrow b_0\pm 0}\xi(b)=\lim_{b\rightarrow b_0\pm 0}\bar{\xi}(b)$ holds. 
In the below, we prove $\sup_b\bar{\xi}(b)=\sup_b\xi(b)$. 

Suppose $\lim_{n\rightarrow\infty}\xi(b_n)=\sup_b\xi(b)$. 
When $b_n$ is bounded, 
there exists $b_0$ such that a subsequence of $b_n$ converges $b_0$. 
Then, $\lim_{b\rightarrow b_0-0}\xi(b)=\sup_b\xi(b)$ or
$\lim_{b\rightarrow b_0+0}\xi(b)=\sup_b\xi(b)$ holds. This means $(0\leq) \sup_b\xi(b)\leq \sup_b\bar{\xi}(b)$. 
Suppose $\lim_{n\rightarrow\infty}\bar{\xi}(b_n')=\sup_b\bar{\xi}(b)$. When $b_n'$ is unbounded, 
$\sup_b\bar{\xi}(b)=0$ should hold. Hence $\sup_b\xi(b)=\sup_b\bar{\xi}(b)=0$. 
When $b_n'$ is bounded, there exists $b_0'$ such that a subsequence of $b_n'$ convergences $b_0'$. Then, 
$\lim_{b\rightarrow b_0'-0}\bar{\xi}(b)=\sup_b\bar{\xi}(b)$ or
$\lim_{b\rightarrow b_0'+0}\bar{\xi}(b)=\sup_b\bar{\xi}(b)$ holds. 
This means that $\sup_b\bar{\xi}(b)\leq \sup_b\xi(b)$. 
As a result, when $b_n$ is bounded, we have $\sup_b\bar{\xi}(b)=\sup_b\xi(b)$. 

If $b_n$ is unbounded, we have $\sup_b\xi(b)=0$ and $Q\big(u(x)<b \big)-P\big(u(x)<b \big)\leq 0$ for any $b$. 
Hence, for any 
$b_0$, $Q\big( u(x)\leq b_0 \big)-P\big( u(x)\leq b_0 \big)=\lim_{b\rightarrow b_0+0}Q\big( u(x)<b \big)-P\big( u(x)<b \big)\leq 0$. 
Thus, we have $\sup_b\bar{\xi}(b)=0$. 

In any case we have $\sup_b\xi(b)=\sup_b \bar{\xi}(b)$. Thus, 
\begin{align*}
&\phantom{=}\sup_{u\in\mathcal{F},b}P(u(x)\geq b)-Q(u(x)\geq b)\\
& =
\sup_{u\in\mathcal{F},b}Q(u(x)<b)-P(u(x)<b) \\
&=
\sup_{u\in\mathcal{F},b}Q(u(x)\leq b)-P(u(x)\leq b) \\ 
&=
\sup_{u\in\mathcal{F},b}Q(-u(x)\geq -b)-P(-u(x)\geq -b)\\
&=
\sup_{u\in\mathcal{F},b}Q(u(x)\geq b)-P(u(x)\geq b). 
\end{align*}
This means that the modulus is no need
%symmetric property of $\mathrm{TV}_{\sigma_0,r}'(P,Q)$. 
\end{proof}

\subsection{Proof of Lemma~\ref{lemma:smoothed-TV}}
\begin{proof}
 The equality $\sigma(z)+\sigma(-z)=1$ in Assumption~(B) leads to  
 $\mathbb{E}_{P}[\sigma(u(X)-b)]-\mathbb{E}_{Q}[\sigma(u(X)-b)] =
 \mathbb{E}_{Q}[\sigma(-u(X)+b)]-\mathbb{E}_{P}[\sigma(-u(X)+b)]$. 
 Assumption~(A) guarantees that $-u\in\mathcal{U}$. Hence, we can drop the modulus in the definition of the STV distance. 
\end{proof}

\subsection{Proof of Lemma~\ref{lemma:RKHS_smoothed-TV}}
\begin{proof}
 For any measurable set $A\in\mathcal{B}$, 
 the indicator function $\1_A\in L_1(\mathcal{X})$ 
 is approximated by the series of continuous functions, $\{f_m\}\subset C(\mathcal{X})$, i.e., 
 $\|\1_A-f_m\|_1\rightarrow0\,(m\rightarrow\infty)$. 
 Here, $L_1$ denotes the set of absolute integrable functions for any probability measure $P\in\mathcal{P}$. 
 We assume $f_m\in[0,1]$ by clipping $f_m$ by $0$ and $1$. 
 Let us consider the approximation of $f_m$ by $\sigma(u-b)$ with some 
 $u\in\mathcal{H}$ and $b\in\mathbb{R}$. 
 Since $\mathcal{H}$ is the universal RKHS, 
 there exist $\{u_{m,k}\}_k\subset\mathcal{H}$ and  $\{b_{m,k}\}_k\subset\Rbb$ such that 
 $\|f_m-\sigma(u_{m,k}-b_{m,k})\|_\infty\rightarrow0,\,(k\rightarrow\infty)$ for each $m$. 
 As a result, one can prove that $\|\1_A-\sigma(u_{m,k_m}-b_{m,k_m})\|_1\rightarrow0\,(m\rightarrow\infty)$. 
 Hence, 
 \begin{align*}
  \int\1_A \mathrm{d}(P-Q)=\lim_{m\rightarrow\infty}\int \sigma(u_{m,k_m}-b_{m,k_m}) \mathrm{d}(P-Q)
  \leq \mathrm{STV}_{\mathcal{H},\sigma}(P,Q)
 \end{align*}
 for any $A$, meaning that $\mathrm{TV}(P,Q)\leq \mathrm{STV}_{\mathcal{H},\sigma}(P,Q)$. 
 On the other hand, $\mathrm{TV}(P,Q)$ has the expression
 $\mathrm{TV}(P,Q)=\sup_{u\in L_0, 0\leq u\leq 1}\lvert\int u \mathrm{d}(P-Q)\rvert\geq \mathrm{STV}_{\mathcal{H},\sigma}(P,Q)$. 
 Therefore, the equality $\mathrm{STV}_{\mathcal{H},\sigma}(P,Q)=\mathrm{TV}(P,Q)$ holds. 

 Let us prove another equality. 
 Suppose that the sequence $\{(u_k,b_k)\}_k\subset\mathcal{H}\times\Rbb$ satisfies 
 \begin{align*}
  \mathbb{E}_{P}[\sigma(u_k-b_k)]-\mathbb{E}_{Q}[\sigma(u_k-b_k)]
  \longrightarrow
  \mathrm{STV}_{\mathcal{H},\sigma}(P,Q)=\mathrm{TV}(P,Q),\quad k\rightarrow\infty. 
 \end{align*}
 Let us define $r_k = k\vee \max\{\|u_{k'}\|\,:\, k'=1,2,\ldots,k\}$. Then, we have 
 \begin{align*}
 \mathbb{E}_{P}[\sigma(u_k-b_k)]-\mathbb{E}_{Q}[\sigma(u_k-b_k)]
  \leq 
  \mathrm{STV}_{\mathcal{H}_{r_k},\sigma}(P,Q)\leq \mathrm{TV}(P,Q). 
 \end{align*}
 As $r_k\nearrow\infty$, we see that 
 $\lim_{r\rightarrow\infty}\mathrm{STV}_{\mathcal{H}_r,\sigma}(P,Q)=\mathrm{TV}(P,Q)$. 
\end{proof}

\subsection{Proof of Proposition~\ref{prop:TV-STV_bias_bound}}
\begin{proof}
 The non-negativity is confirmed from the property of $\mathrm{STV}$ with $\sigma:\Rbb\rightarrow[0,1]$. 
 Below, we derive the upper bound. 
 Due to Lemma~\ref{lemma:smoothed-TV}, we can drop the absolute value in the STV distance.
 Let us define $p=\frac{\rmd{P}}{\rmd{\mu}}$ and $q=\frac{\rmd{Q}}{\rmd{\mu}}$. Then, we have the following inequalities: 
\begin{align*}
&\phantom{=} 
 \mathrm{TV}(P, Q) - \mathrm{STV}_{c\,\mathcal{U}, \sigma}(P,Q)\\
 &\leq 
 \mathbb{E}_{P}[\1[s(X)\geq0]] - \mathbb{E}_{Q}[\1[s(X)\geq0]]
  -
 (\mathbb{E}_{P}[\sigma(c s(X))] - \mathbb{E}_{Q}[\sigma(c s(X))])\\
 & =
 \int \left\{ \1[s(x)\geq0] - \sigma(c s(x)) \right\} (p(x)-q(x)) \rmd\mu(x)\\
 & =
 \int_{p\geq q}\sigma\left(-c\log\frac{p(x)}{q(x)}\right)\left(\frac{p(x)}{q(x)}-1\right)q(x)\mathrm{d}\mu(x)\\
 &\phantom{=}
 + \int_{q>p}\sigma\left(-c\log\frac{q(x)}{p(x)}\right)\left(\frac{q(x)}{p(x)}-1\right)p(x)\mathrm{d}\mu(x).
\end{align*}
 Since the decay rate of $\sigma$ is $\lambda(c)$, the inequality
 \begin{align*}
 \mathrm{TV}(P, Q) - \mathrm{STV}_{c\,\mathcal{U}, \sigma}(P,Q)
  &\leq 
  \lambda(c)  \left\{  \int_{p\geq q}q(x)\mathrm{d}\mu(x)  + \int_{q>p}p(x)\rmd \mu(x) \right\}\\
  &= 
  \lambda(c) \int \min\{p, q\} \rmd\mu = \lambda(c) (1-\mathrm{TV}(P,Q))
 \end{align*}
 holds for $c> C_0$. 
\end{proof}

\subsection{Proof of Lemma~\ref{lemma:lower_bound_DecayRate}}
\begin{proof}
Let $y=c\log t$, then decay rate $\lambda(c)$ should satisfy 
\begin{align*}
 \sigma(-y)
 \leq 
 \frac{\lambda(c)}{e^{y/c}-1} =\frac{\lambda(c)}{y/c+\frac{1}{2!}(y/c)^2+\cdots}
 = 
 \frac{\lambda(c)c}{y+\frac{1}{2!}y^2/c+\cdots}. 
\end{align*}
Suppose $\lambda(c)=o(1/c)$. For a fixed $y>0$, we have 
 $\lim_{c\rightarrow\infty}\frac{\lambda(c)c}{y+\frac{1}{2!}y^2/c+\cdots}=0$. Hence, 
 the function $\sigma(z)$ should be $0$ on $z<0$ and thus, $\sigma(z)=1$ for $z>0$. 
 This contradicts Assumption~(B). 
\end{proof}

\subsection{Some formulae for Kernel Exponential Family}
\label{app:formulae_KEF}
\begin{lemma}
\label{lemma:KEF-basic}
 Suppose $\sup_{x\in\mathcal{X}}k(x,x)\leq K^2<\infty$ and $p_f\in\mathcal{P}_\mathcal{H}$. 
 Then, 
 \begin{align*}
  \lim_{\varepsilon\rightarrow0}\frac{A(f+\varepsilon u)-A(f)}{\varepsilon}=\mathbb{E}_{P_f}[u]
 \end{align*}
 holds for any $u\in\mathcal{H}$. Furthermore,
 \begin{align}
  \label{Af_ineq}
 &\<\mathbb{E}_{P_g}[k(x,\cdot)], f-g\>\leq A(f)-A(g)\leq \<\mathbb{E}_{P_f}[k(x,\cdot)], f-g\>,\\ 
  \label{Af_Ag_upperbound}
  &\abs{A(f)-A(g)}\leq K\|f-g\|
 \end{align}
 hold for $f,g\in\mathcal{H}$ such that $A(f)$ and $A(g)$ are finite. 
\end{lemma}

%Proof of Lemma~\ref{lemma:KEF-basic}}
\begin{proof}
 Without loss of generality,  we assume that $K\leq 1$. 
 For $f\in\mathcal{H}$ such that $A(f)<\infty$ and $u\in\mathcal{H}$, 
 we have 
 $\lvert\frac{d}{d\varepsilon}e^{f(x)+\varepsilon u(x)}\rvert=\lvert e^{f(x)+\varepsilon u(x)}u(x)\rvert\leq e^{f(x)} e^{\lvert \varepsilon\rvert\|u\|_\infty}\|u\|_\infty$, 
 which is integrable. 
 Thus, the equality 
$\frac{\mathrm{d}}{\mathrm{d}\varepsilon}\int e^{f+\varepsilon u}\mathrm{d}\mu=
 \int \frac{\mathrm{d}}{\mathrm{d}\varepsilon} e^{f+\varepsilon u}\mathrm{d}\mu=\int e^{f+\varepsilon
 u}u\mathrm{d}\mu$ leads to 
 \begin{align*}
  \frac{\mathrm{d}}{\mathrm{d}\varepsilon}
  \log\int e^{f(x)+\varepsilon u(x)}\mathrm{d}\mu\bigg\lvert_{\varepsilon=0}
  =\frac{\int e^{f(x)}u(x)\mathrm{d}\mu}{\int e^{f(x)}\mathrm{d}\mu}%-\<f,u\>
  =\mathbb{E}_{P_f}[u].%-\<f,u\>. 
 \end{align*}
 The second statement is obtained as follows. 
 The convexity of $A(f)$ in $f\in\mathcal{H}$ leads to 
 \begin{align*}
  A(f)-A(g)\geq \<\mathbb{E}_{P_g}[k(x,\cdot)], f-g\>,\, \text{and}\, A(g)-A(f)\geq \<\mathbb{E}_{P_f}[k(x,\cdot)], g-f\>. 
 \end{align*}
 Hence, we have 
 \begin{align*}
  \<\mathbb{E}_{P_g}[k(x,\cdot)], f-g\>\leq A(f)-A(g)\leq \<\mathbb{E}_{P_f}[k(x,\cdot)], f-g\>.
 \end{align*}
 Therefore, $\abs{A(f)-A(g)}\leq \|f-g\|$. 
\end{proof}

\begin{remark}
Suppose that the kernel function is not bounded. Then, 
$\lim_{\varepsilon\rightarrow0}\frac{A(f+\varepsilon u)-A(f)}{\varepsilon}=\mathbb{E}_{P_f}[u]$ 
holds if $e^{f(x)+\varepsilon u(x)}u(x)$ is integrable w.r.t. $\mu$ for arbitrary $\varepsilon$ in the
 vicinity of zero. The inequalities \eqref{Af_ineq} and \eqref{Af_Ag_upperbound} are replaced with 
 $\mathbb{E}_{P_g}[f(x)-g(x)]\leq A(f)-A(g)\leq \mathbb{E}_{P_f}[f(x)-g(x)]$ 
 and 
 $\abs{A(f)-A(g)}\leq K_{f,g}\|f-g\|$, where 
 $K_{f,g}=\max\{\mathbb{E}_{P_g}[\sqrt{k(x,x)}],\,\mathbb{E}_{P_f}[\sqrt{k(x,x)}]\}$. 
\end{remark}

\subsection{Proof of Lemma~\ref{lemma:kexp_TV}}
\begin{proof}
% Without loss of generality, let us define $U=1$. 
 From the definition, we have 
\begin{align*}
\mathrm{STV}_{\mathcal{H}_U,\1}(P_f, P_g)
 =
 \sup_{\|u\|\leq U,b\in\Rbb} P_f(u(x)\geq b)-P_g(u(x)\geq b) \leq \mathrm{TV}(P_f,P_g). 
\end{align*}
 We prove the inequality of the converse direction. 
 The TV distance is expressed by the $L_1$ distance between probability densities. 
 For $f\neq g$, 
 \begin{align*}
  2\mathrm{TV}(P_f,P_g)
  &=
  \int_{\mathcal{X}} \abs{p_f-p_g}\mathrm{d}\mu\\
  &=
  \int_{p_f\geq p_g} (p_f-p_g)\mathrm{d}\mu +   \int_{p_g\geq p_f} (p_g-p_f)\mathrm{d}\mu \\
  &= 
  \int_{\frac{U(f-g)}{\|f-g\|}\geq \frac{U(A(f)-A(g))}{\|f-g\|}} (p_f-p_g)\mathrm{d}\mu +   
  \int_{\frac{U(g-f)}{\|g-f\|}\geq \frac{U(A(g)-A(f))}{\|g-f\|}} (p_g-p_f)\mathrm{d}\mu \\
  &\leq 2\mathrm{STV}_{\mathcal{H}_U,\1}(P_f,P_g). 
 \end{align*}
 Hence, we have $\mathrm{STV}_{\mathcal{H}_U,\1}(P_f,P_g)=\mathrm{TV}(P_f,P_g)$. 
\end{proof}

\subsection{Proof of Lemma~\ref{lemma:limitSTV-KEF_TV}}
\begin{proof}
 When $f=g$, the lemma holds by regarding $\lambda(U/0)=\lambda(\infty)=0$. 
 Suppose that $f\neq g$. Let us define $s(x)=\{f(x)-g(x)-(A(f)-A(g))\}/\|f-g\|$. 
 Then, we have
\begin{align*}
\mathrm{TV}(P_f,P_g)  
& = \frac{1}{2}\int_{\mathcal{X}}\abs{p_f(x)-p_g(x)}\mathrm{d}\mu \\
& =  
 \frac{1}{2}\int_{s\geq0} (p_f(x)-p_g(x))\mathrm{d}\mu +\frac{1}{2}\int_{s<0} (p_g(x)-p_f(x))\mathrm{d}\mu \\
&=
 \frac{1}{2}\int_{s\geq0} (p_f(x)-p_g(x))\mathrm{d}\mu \\
 %%%%%%%%%
&\phantom{=} +\frac{1}{2}\underbrace{\int_{\mathcal{X}} (p_g(x)-p_f(x))\mathrm{d}\mu}_{=0}
  -\frac{1}{2}\int_{s\geq 0}   (p_g(x)-p_f(x))\mathrm{d}\mu \\
& =  \int_{s\geq0} (p_f(x)-p_g(x))\mathrm{d}\mu \\
& = \mathbb{E}_{P_f}[\1[U s(x)\geq0]]-\mathbb{E}_{P_g}[\1[U s(x)\geq0]], 
\end{align*}
where $U$ is a positive number. 
We have 
\begin{align*}
 \mathbb{E}_{P_f}[\sigma(U s(X))]-\mathbb{E}_{P_g}[\sigma(U s(X))]
 \leq 
 \mathrm{STV}_{\mathcal{H}_U,\sigma}(P_f,P_g)
 \leq 
 \mathrm{TV}(P_f,P_g). 
\end{align*}
 Note that 
\begin{align*}
 \lim_{U\rightarrow\infty}\sigma(U z)
 =
 \1[z\geq0]-\frac{1}{2}\1[z=0]=
 \begin{cases}
  1, & z>0,\\ 1/2, & z=0,\\ 0, & z<0. 
 \end{cases}
\end{align*}
Lebesgue's dominated convergence theorem leads to 
\begin{align*}
&\phantom{=} \lim_{U\rightarrow\infty}\mathbb{E}_{P_f}[\sigma(U s(X))]-\mathbb{E}_{P_g}[\sigma(U s(X))] \\
& =
 \mathbb{E}_{P_f}[\1[s(X)\geq 0]]-\mathbb{E}_{P_g}[\1[s(X)\geq 0]]
 -\frac{1}{2}\underbrace{\{\mathbb{E}_{P_f}[\1[s(X)=0]]-\mathbb{E}_{P_g}[\1[s(X)=0]]\}}_{=\int_{\{p_f=p_g\}}(p_f-p_g)\mathrm{d}\mu=0}\\
& = \mathrm{TV}(P_f,P_g). 
\end{align*}
As a result, we have 
$\lim_{U\rightarrow\infty}\mathrm{STV}_{\mathcal{H}_U, \sigma}(P_f,P_g) =
 \mathrm{TV}(P_f,P_g)$. 

The proof of the inequality follows Proposition~\ref{prop:TV-STV_bias_bound}. 
From $P_f\ll \mu$  and $P_g\ll\mu$, we have $s(x)=\log\frac{p_f(x)}{p_g(x)}=f(x)-g(x)-(A(f)-A(g))$. 
When $f\neq g$ and $U>\|f-g\|$, we have 
\begin{align*}
 0\leq \mathrm{TV}(P_f,P_g)-\mathrm{STV}_{\mathcal{H}_{U},\sigma}(P_f,P_g)
\leq 
 \lambda\bigg(\frac{U}{\|f-g\|}\bigg)(1-\mathrm{TV}(P_f,P_g)), 
\end{align*}
 where $\mathcal{H}_U$ is regarded as $\mathcal{H}_{\frac{U}{\|f-g\|}\|f-g\|}$. 
\end{proof}

\section{Proofs of Theorems and Lemmas in Section~\ref{sec:PredictionErrorBounds_Finite-dim_RKHS}}
\label{app:proof-learning}

\subsection{Proof of Theorem~\ref{thm:convergence-rate}}
\begin{proof}
 For the probability distributions $P_{f_0},P_{\widehat{f}_n}$ in $\mathcal{P}_{\mathcal{H}}$, we have 
\begin{align*}
 \mathrm{TV}(P_{f_0},P_{\widehat{f}_n}) 
& = 
 \mathrm{STV}_{\widetilde{\mathcal{H}}_U,\1}(P_{f_0},P_{\widehat{f}_n})   \\
& \leq 
 \mathrm{STV}_{\widetilde{\mathcal{H}}_U,\1}(P_{f_0},P_n) + \mathrm{STV}_{\widetilde{\mathcal{H}}_U,\1}(P_{\widehat{f}_n},P_n)   \\
 & \leq 
 2 \mathrm{STV}_{\widetilde{\mathcal{H}}_U,\1}(P_{f_0},P_n)\\
& \leq  
 2 \mathrm{STV}_{\widetilde{\mathcal{H}}_U,\1}(P_{f_0}, P_\varepsilon)+ 2 \mathrm{STV}_{\widetilde{\mathcal{H}}_U,\1}(P_\varepsilon,P_n)\\
& \leq  
 2 \mathrm{TV}(P_{f_0}, P_\varepsilon) + 2 \mathrm{STV}_{\widetilde{\mathcal{H}}_U,\1}(P_\varepsilon,P_n)\\
& \leq  
 2\varepsilon + 2 \mathrm{STV}_{\widetilde{\mathcal{H}}_U,\1}(P_\varepsilon,P_n). 
\end{align*} 
Lemma~\ref{lemma:kexp_TV} is used in the first equality. 
The probabilistic upper bound of $\mathrm{STV}_{\widetilde{\mathcal{H}}_U,\1}(P_\varepsilon,P_n)$ is evaluated by the classical
VC theory~\cite{book:Vapnik:1998,mohri18:_found_machin_learn}. 
The VC-dimension of the function set $\mathcal{F}=\{x\mapsto\1[u(x)\geq b]\,:\,u\in\widetilde{\mathcal{H}}_U, b\in\Rbb\}$ 
is bounded above by the dimension of $\widetilde{\mathcal{H}}$~\cite{mohri18:_found_machin_learn,wainwright19:_high}. 
Let us define 
\begin{align*}
 \displaystyle g(X_1,\ldots,X_n)
 :=
 \mathrm{STV}_{\widetilde{\mathcal{H}}_U,\1}(P_\varepsilon,P_n)
 =
 \sup_{\substack{u\in\widetilde{\mathcal{H}}_U,\|u\|\leq r,b\in \Rbb}}
 P_{\varepsilon}(u(X)\geq b)- P_n(u(X)\geq b), 
\end{align*}
McDiarmid inequality leads to inequality, 
\begin{align*}
 \mathrm{Pr}\bigg(g \leq \mathbb{E}[g] + \sqrt{\frac{\log(1/\delta)}{2n}}\bigg)  \geq 1-\delta. 
\end{align*}
 The expectation $\mathbb{E}[g]$ is bounded above by the Rademacher complexity $\mathfrak{R}_n(\mathcal{F})$
 which is bounded above by Dudley's integral entropy bound, 
\begin{align*}
 \mathbb{E}[g]
 &\leq
 \mathfrak{R}_n(\mathcal{F})
 \leq 
 \mathbb{E}\left[
 \frac{24}{\sqrt{n}}
 \int_0^2\sqrt{\log
 N(\delta,\mathcal{F},\|\cdot\|_n)}\mathrm{d}\delta
 \right], 
\end{align*}
 where $N(\delta,\mathcal{F},\|\cdot\|_n)$ is the covering number of $\mathcal{F}$
 with the norm $\|f\|_n=\sqrt{\frac{1}{n}\sum_{i=1}^{n}f(X_i)^2}$. 
 The covering number of $\mathcal{F}$ is bounded above by
 $(3/\delta)^{\widetilde{d}}$ as $\mathcal{F}$ is a class of 1-uniformly bounded functions 
 \cite[Example~5.24]{wainwright19:_high}. The above integral is then bounded above by $24\times3\sqrt{\widetilde{d}/n}$. 
As a result, we have 
\begin{align*}
 \mathrm{TV}(P_{f_0},P_{\widehat{f}_n}) \lesssim \varepsilon +  \sqrt{\frac{\widetilde{d}}{n}}
 +\sqrt{\frac{\log(1/\delta)}{n}}
\end{align*}
with probability greater than $1-\delta$. 
\end{proof}

\subsection{Proof of Theorem~\ref{thm:reg-STV-learning}}
\begin{proof}
First, let us consider the convergence rate of $\mathrm{TV}(P_{f_0}, P_{\widehat{f}_r})$. 
The $\mathrm{TV}$ distance between the target distribution and the estimator is bounded above as follows. 
For a sufficiently large $r$, suppose $f_0\in\mathcal{H}_{r}$. Then, 
\begin{align*}
 \mathrm{STV}_{\mathcal{H}_U,\sigma}(P_{f_0}, P_{\widehat{f}_r})
 & \leq
 \mathrm{STV}_{\mathcal{H}_U,\sigma}(P_{f_0}, P_n)
 +
 \mathrm{STV}_{\mathcal{H}_U,\sigma}(P_n, P_{\widehat{f}_r})\\
& \leq 
 2\mathrm{STV}_{\mathcal{H}_U,\sigma}(P_{f_0}, P_n)\\
& \leq 
 2\mathrm{STV}_{\mathcal{H}_U,\sigma}(P_{f_0},P_\varepsilon) 
 + 
 2\mathrm{STV}_{\mathcal{H}_U,\sigma}(P_\varepsilon,P_n) \\
& \leq 
 2\mathrm{TV}(P_{f_0},P_\varepsilon) 
 + 
 2\mathrm{STV}_{\mathcal{H}_U,\sigma}(P_\varepsilon,P_n) \\
& \leq 
 2\varepsilon +  2\mathrm{STV}_{\mathcal{H}_U,\sigma}(P_\varepsilon,P_n)\\
& \lesssim  
 \varepsilon +  \sqrt{\frac{d}{n}} + \sqrt{\frac{\log(1/\delta)}{n}}
\end{align*}
with probability $1-\delta$. Therefore, 
\begin{align*}
 \mathrm{TV}(P_{f_0}, P_{\widehat{f}_r}) 
 & = 
 \left\{\mathrm{TV}(P_{f_0}, P_{\widehat{f}_r})-\mathrm{STV}_{\mathcal{H}_U,\sigma}(P_{f_0}, P_{\widehat{f}_r})\right\}
 +\mathrm{STV}_{\mathcal{H}_U,\sigma}(P_{f_0}, P_{\widehat{f}_r})\\
 & \leq 
 \frac{\|f_0-\widehat{f}_r\|}{U} +\mathrm{STV}_{\mathcal{H}_U,\sigma}(P_{f_0}, P_{\widehat{f}_r})
\lesssim
 \frac{r}{U}+ \varepsilon+\sqrt{\frac{d}{n}}
\end{align*}
with high probability, where $U\geq 2r\geq \|f_0-\widehat{f}_r\|$. 
As shown above, the bias term in \eqref{eqn:Bias_term_abstract} is bounded by $\|f_0-\widehat{f}_r\|/U$ for the regularized STV learning. 

Next, let us consider the convergence rate of $\mathrm{TV}(P_{f_0}, P_{\widehat{f}_{\mathrm{reg},r}})$. 
The standard argument works to derive the upper bound as follows: 
\begin{align*}
&\phantom{\leq} \mathrm{STV}_{\mathcal{H}_U,\sigma}(P_{f_0}, P_{\widehat{f}_{\mathrm{reg},r}})\\
 & \leq
 \mathrm{STV}_{\mathcal{H}_U,\sigma}(P_{f_0}, P_{\widehat{f}_{\mathrm{reg},r}})\\
 &\phantom{ \leq}+
 \left[\mathrm{STV}_{\mathcal{H}_U,\sigma}(P_{n}, P_{\widehat{f}_{\mathrm{reg},r}}) + \frac{1}{r^2}\|\widehat{f}_{\mathrm{reg},r}\|^2\right]
 -
 \left[\mathrm{STV}_{\mathcal{H}_U,\sigma}(P_{n}, P_{f_0}) + \frac{1}{r^2}\|f_0\|^2\right]\\
 &\phantom{ \leq}
 +\mathrm{STV}_{\mathcal{H}_U,\sigma}(P_{n}, P_{f_0})
 -\mathrm{STV}_{\mathcal{H}_U,\sigma}(P_{n}, P_{\widehat{f}_{\mathrm{reg},r}})
 +\frac{1}{r^2}\|f_0\|^2\\
& \leq 
 \mathrm{STV}_{\mathcal{H}_U,\sigma}(P_{f_0}, P_{\widehat{f}_{\mathrm{reg},r}})
 +\mathrm{STV}_{\mathcal{H}_U,\sigma}(P_{n}, P_{f_0})
 -\mathrm{STV}_{\mathcal{H}_U,\sigma}(P_{n}, P_{\widehat{f}_{\mathrm{reg},r}})
 +\frac{1}{r^2}\|f_0\|^2\\
& \leq 
 2\mathrm{STV}_{\mathcal{H}_U,\sigma}(P_{n}, P_{f_0})+\frac{1}{r^2}\|f_0\|^2\\
& \leq 
 2\mathrm{TV}(P_{f_0},P_\varepsilon) 
 + 
 2\mathrm{STV}_{\mathcal{H}_U,\sigma}(P_\varepsilon,P_n)+\frac{1}{r^2}\|f_0\|^2 \\ 
& \leq 
 2\varepsilon +  2\mathrm{STV}_{\mathcal{H}_U,\sigma}(P_\varepsilon,P_n)+\frac{1}{r^2}\|f_0\|^2 \\
& \lesssim
 \varepsilon +  \sqrt{\frac{d}{n}} + \sqrt{\frac{\log(1/\delta)}{n}} + \frac{\|f_0\|^2}{r^2}. 
\end{align*}
The first inequality is obtained by the definition of $\widehat{f}_{\mathrm{reg},r}$. 
The triangle inequality is used in the second inequality. 
Thus, we obtain
\begin{align*}
 \mathrm{TV}(P_{f_0}, P_{\widehat{f}_{\mathrm{reg},r}}) 
 & = 
 \left\{\mathrm{TV}(P_{f_0}, P_{\widehat{f}_{\mathrm{reg},r}})
 -\mathrm{STV}_{\mathcal{H}_U,\sigma}(P_{f_0}, P_{\widehat{f}_{\mathrm{reg},r}})\right\}
 +\mathrm{STV}_{\mathcal{H}_U,\sigma}(P_{f_0}, P_{\widehat{f}_{\mathrm{reg},r}})\\
 & \leq 
 \frac{\|f_0-\widehat{f}_{\mathrm{reg},r}\|}{U} +\mathrm{STV}_{\mathcal{H}_U,\sigma}(P_{f_0}, P_{\widehat{f}_{\mathrm{reg},r}})\\
 &\lesssim
 \frac{r}{U}+ \varepsilon+\sqrt{\frac{d}{n}}+ \sqrt{\frac{\log(1/\delta)}{n}} +
 \frac{\|f_0\|^2}{r^2}
\end{align*}
for $U\geq 2r\geq \|f_0-\widehat{f}_{\mathrm{reg},r}\|$. 
\end{proof}

\subsection{Proof of Theorem~\ref{thm:fullreg-STV-learning}}
\begin{proof}
For $\widetilde{U}\leq{U}$, let us consider an upper bound of
$\mathrm{STV}_{\mathcal{H}_{\widetilde{U}},\sigma}(P_{n},P_{\check{f}_{\mathrm{reg},r}})$. 
In the following, let $\check{f}$ be $\check{f}_{\mathrm{reg},r}$ for the sake of simplicity. 
The inequality $\|u\|^2\leq \widetilde{U}^2$ for $u\in\mathcal{H}_{\widetilde{U}}$ leads to 
\begin{align*}
&\phantom{\leq} \sup_{u\in\mathcal{H}_{\widetilde{U}},b\in\Rbb} 
 \mathbb{E}_{P_n}[\sigma(u(X)-b)]-\mathbb{E}_{P_{\check{f}}}[\sigma(u(X)-b)] \\
&\leq 
\sup_{u\in\mathcal{H}_{\widetilde{U}},b\in\Rbb} 
 \left\{\mathbb{E}_{P_n}[\sigma(u(X)-b)]-\mathbb{E}_{P_{\check{f}}}[\sigma(u(X)-b)]
 -\frac{\|u\|^2}{U^2}+\frac{\|\check{f}\|^2}{r^2}\right\} + \frac{\widetilde{U}^2}{U^2} \\
&\leq 
\sup_{u\in\mathcal{H},b\in\Rbb}
 \left\{\mathbb{E}_{P_n}[\sigma(u(X)-b)]-\mathbb{E}_{P_{\check{f}}}[\sigma(u(X)-b)]
 -\frac{\|u\|^2}{U^2}+\frac{\|\check{f}\|^2}{r^2}\right\} + \frac{\widetilde{U}^2}{U^2}. 
\end{align*}
Moreover, we have
\begin{align}
 \label{eqn:app-u-norm-bound}
 &\phantom{\leq}
 \sup_{u\in\mathcal{H},b\in\Rbb}\mathbb{E}_{P_n}[\sigma(u(X)-b)]-\mathbb{E}_{P_{f_0}}[\sigma(u(X)-b)]-\frac{\|u\|^2}{U^2}\\
& \leq 
 \sup_{u\in\mathcal{H}_U,b\in\Rbb}\mathbb{E}_{P_n}[\sigma(u(X)-b)]-\mathbb{E}_{P_{f_0}}[\sigma(u(X)-b)]\nonumber
\end{align}
since the norm of the optimal $u\in\mathcal{H}$ in \eqref{eqn:app-u-norm-bound} is less than or equal to $U$. 
Therefore, we have
\begin{align*}
& \phantom{=}\mathrm{STV}_{\mathcal{H}_{\widetilde{U}},\sigma}(P_{n},P_{\check{f}}) \\
&=
\sup_{u\in\mathcal{H}_{\widetilde{U}},b\in\Rbb}
 \mathbb{E}_{P_n}[\sigma(u(X)-b)]-\mathbb{E}_{P_{\check{f}}}[\sigma(u(X)-b)] \\
&\leq 
\sup_{u\in\mathcal{H},b\in\Rbb}
 \left\{\mathbb{E}_{P_n}[\sigma(u(X)-b)]-\mathbb{E}_{P_{\check{f}}}[\sigma(u(X)-b)]
 -\frac{\|u\|^2}{U^2}+\frac{\|\check{f}\|^2}{r^2}\right\} + \frac{\widetilde{U}^2}{U^2} \\
&\phantom{\leq}
 -
\sup_{u\in\mathcal{H},b\in\Rbb}
 \left\{\mathbb{E}_{P_n}[\sigma(u(X)-b)]-\mathbb{E}_{P_{f_0}}[\sigma(u(X)-b)]
 -\frac{\|u\|^2}{U^2}+\frac{\|f_0\|^2}{r^2}\right\}\\
&\phantom{\leq}+\sup_{u\in\mathcal{H}_{U},b\in\Rbb}
 \left\{\mathbb{E}_{P_n}[\sigma(u(X)-b)]-\mathbb{E}_{P_{f_0}}[\sigma(u(X)-b)]\right\}
 +\frac{\|f_0\|^2}{r^2}\\
&\leq
\mathrm{STV}_{\mathcal{H}_{U},\sigma}(P_n,P_{f_0})+\frac{\widetilde{U}^2}{U^2} +\frac{\|f_0\|^2}{r^2}. 
\end{align*}
The second inequality is obtained from the optimality of $\check{f}$. 
Hence, we have 
\begin{align*}
 \mathrm{STV}_{\mathcal{H}_{\widetilde{U}},\sigma}(P_{f_0},P_{\check{f}})
& \leq 
 \mathrm{STV}_{\mathcal{H}_{\widetilde{U}},\sigma}(P_n,P_{f_0}) +\mathrm{STV}_{\mathcal{H}_{\widetilde{U}},\sigma}(P_n,P_{\check{f}}) \\
& \leq 
 2\cdot\mathrm{STV}_{\mathcal{H}_{U},\sigma}(P_n,P_{f_0}) +\frac{\widetilde{U}^2}{U^2} +\frac{\|f_0\|^2}{r^2}. 
\end{align*}
Therefore, we obtain 
\begin{align*}
 \mathrm{TV}(P_{f_0}, P_{\check{f}})
 & = 
 \left\{\mathrm{TV}(P_{f_0}, P_{\check{f}}) -\mathrm{STV}_{\mathcal{H}_{\widetilde{U}},\sigma}(P_{f_0},P_{\check{f}}) \right\}
 +\mathrm{STV}_{\mathcal{H}_{\widetilde{U}},\sigma}(P_{f_0},P_{\check{f}})\\
 & \leq
\frac{\|f_0-\check{f}\|}{\widetilde{U}} +% \frac{2r}{\widetilde{U}} +
 2\cdot\mathrm{STV}_{\mathcal{H}_{U},\sigma}(P_n,P_{f_0}) +\frac{\widetilde{U}^2}{U^2} +\frac{\|f_0\|^2}{r^2}\\
 &\lesssim
 \frac{r}{\widetilde{U}}+\frac{\widetilde{U}^2}{U^2} +\frac{\|f_0\|^2}{r^2} 
 +\mathrm{TV}(P_{f_0},P_\varepsilon)+\mathrm{STV}_{\mathcal{H}_U,\sigma}(P_n,P_\varepsilon)\\
 &\lesssim
 \frac{r}{\widetilde{U}}+\frac{\widetilde{U}^2}{U^2} +\frac{\|f_0\|^2}{r^2} 
 +\varepsilon + \sqrt{\frac{d}{n}}+\sqrt{\frac{\log(1/\delta)}{n}}
\end{align*}
for $U\geq \widetilde{U}\geq 2r$. By setting $\widetilde{U}=(2r)^{1/3}U^{2/3}$ for $2r\leq U$, 
we obtain the inequality in Theorem~\ref{thm:fullreg-STV-learning}. 
\end{proof}

\subsection{Proof of Theorem~\ref{thm:infinite-STV-learning-bound}}
\label{app:Proof-thm:infinite-STV-learning-bound}

Using the upper bound of $\mathrm{STV}_{\widetilde{H}_U,\sigma}(P_\varepsilon, P_n)$, 
we evaluate the estimation error of the regularized STV learning with infinite-dimensional RKHSs. 
The Lipschitz constant of $\sigma$ is $1$. 
Since the Rademacher complexity of the function set $\{\sigma(u-b)\,:\,u\in\mathcal{H}_U, \abs{b}\leq U\}$ 
over $n$ samples is bounded above by $U/\sqrt{n}$, we have 
\begin{align*}
 \mathrm{STV}_{\widetilde{\mathcal{H}}_U,\sigma}(P_\varepsilon,P_n)
 &=\sup_{u\in\widetilde{\mathcal{H}}_U,\abs{b}\leq U}
 \mathbb{E}_{P_\varepsilon}[\sigma(u(X)-b)]-\mathbb{E}_{P_n}[\sigma(u(X)-b)]\\
& \lesssim \frac{U}{\sqrt{n}} + \sqrt{\frac{\log(1/\delta)}{n}}
\end{align*}
with probability greater than $1-\delta$. 
For a sufficiently large $r$ such that $f_0\in\mathcal{H}_r$, the following inequalities hold, 
\begin{align*}
 \mathrm{STV}_{\widetilde{\mathcal{H}}_U,\sigma}(P_{f_0}, P_{\widehat{f}_r})
 & \leq
 \mathrm{STV}_{\widetilde{\mathcal{H}}_U,\sigma}(P_{f_0}, P_n)
 +
 \mathrm{STV}_{\widetilde{\mathcal{H}}_U,\sigma}(P_n, P_{\widehat{f}_r})\\
& \leq 
 2\mathrm{STV}_{\widetilde{\mathcal{H}}_U,\sigma}(P_{f_0}, P_n)\\
& \leq 
 2\mathrm{STV}_{\widetilde{\mathcal{H}}_U,\sigma}(P_{f_0},P_\varepsilon) 
 + 
 2\mathrm{STV}_{\widetilde{\mathcal{H}}_U,\sigma}(P_\varepsilon,P_n) \\
& \leq 
 2\mathrm{TV}(P_{f_0},P_\varepsilon) 
 + 
 2\mathrm{STV}_{\widetilde{\mathcal{H}}_U,\sigma}(P_\varepsilon,P_n) \\
& \leq 
 2\varepsilon +  2\mathrm{STV}_{\widetilde{\mathcal{H}}_U,\sigma}(P_\varepsilon,P_n)
\lesssim
 \varepsilon +  \frac{U}{\sqrt{n}} + \sqrt{\frac{\log(1/\delta)}{n}}. 
\end{align*}
Therefore, we have 
\begin{align*}
 \mathrm{TV}(P_{f_0}, P_{\widehat{f}_r})
 & = 
 \left\{\mathrm{TV}(P_{f_0}, P_{\widehat{f}_r})
 -\mathrm{STV}_{\widetilde{\mathcal{H}}_U,\sigma}(P_{f_0}, P_{\widehat{f}_r})\right\}
 +\mathrm{STV}_{\widetilde{\mathcal{H}}_U,\sigma}(P_{f_0}, P_{\widehat{f}_r})\\ 
 & \lesssim
 \frac{\|f_0-\widehat{f}_r\|}{U} + \varepsilon +  \frac{U}{\sqrt{n}} + \sqrt{\frac{\log(1/\delta)}{n}}\\
 & \lesssim
 \frac{r}{U} + \varepsilon +  \frac{U}{\sqrt{n}} + \sqrt{\frac{\log(1/\delta)}{n}}. 
\end{align*}

For the estimator $\widehat{f}_{\mathrm{reg},r}$ given by the additive regularization, 
the same argument of the Theorem~\ref{thm:reg-STV-learning} leads to 
\begin{align*}
& \phantom{=}\mathrm{TV}(P_{f_0}, P_{\widehat{f}_{\mathrm{reg},r}})\\
 & = 
 \left\{\mathrm{TV}(P_{f_0}, P_{\widehat{f}_{\mathrm{reg},r}})
-\mathrm{STV}_{\widetilde{\mathcal{H}}_U,\sigma}(P_{f_0}, P_{\widehat{f}_{\mathrm{reg},r}})\right\}
 +\mathrm{STV}_{\widetilde{\mathcal{H}}_U,\sigma}(P_{f_0}, P_{\widehat{f}_r})\\
 & \lesssim
 \frac{\|f_0-\widehat{f}_{\mathrm{reg},r}\|}{U} 
+ \varepsilon +  \frac{U}{\sqrt{n}} + \sqrt{\frac{\log(1/\delta)}{n}} + \frac{\|f_0\|^2}{r^2}\\
 & \lesssim
 \frac{r}{U} + \varepsilon +  \frac{U}{\sqrt{n}} + \sqrt{\frac{\log(1/\delta)}{n}}+ \frac{\|f_0\|^2}{r^2}. 
\end{align*}
Likewise, the same argument of the Theorem~\ref{thm:fullreg-STV-learning}
leads to the estimation error bound of $\check{f}_{\mathrm{reg},r}$. 

\subsection{Proof of Corollary~\ref{cor:infinite-dim-RKHS-optbound}}
For the estimator $\widehat{f}_r$, 
Theorem~\ref{thm:infinite-STV-learning-bound} with $r=O(\mathrm{Poly}(\log n))$ and $U=n^{1/4}$ 
leads to the upper bound $\varepsilon+n^{-1/4}$, where the poly-log order is omitted. 
Likewise, the estimator $\widehat{f}_{\mathrm{reg},r}$ with $r=n^{1/10},\, U=n^{3/10}$ and 
the estimator $\check{f}_{\mathrm{reg},r}$ with $r=n^{1/12}, U=n^{1/3}$
yield the upper bound in the corollary.

\section{Proof in Section~\ref{sec:Accuracy_Par_est}}
\label{app:error_par_est}

\subsection{Proof of Lemma~\ref{lemma:lower-bd-TV_RKHS}}
\label{appendix:proof-lemma_TV_lowerbound}
\begin{proof}
Due to \eqref{Af_ineq} in Appendix~\ref{app:formulae_KEF}, 
we see that there exists $0\leq \beta\leq 1$ such that $A(f)-A(g)=\<f-g,\mathbb{E}_{\beta}[k(x,\cdot)]\>$, where 
$\mathbb{E}_{\beta}$ is the expectation with the probability 
$p_\beta=\exp\{\beta f + (1-\beta) g -A(\beta f + (1-\beta) g)\}\in\mathcal{P}_{\mathcal{H}}$. 
A simple calculation leads that for $r\geq\log(2)/4$ and $a\in\Rbb$ such that $\abs{a}\leq 4r$, the inequality 
\begin{align}
 \label{appendix:exp-inequ} 
\abs{e^{a}-1}\geq \frac{\abs{a}}{8r}
\end{align}
holds. 
For $f\neq g$, let us define $s=(f-g)/\|f-g\|$. For $f,g\in\mathcal{H}_r$, we have
\begin{align*}
\abs{f-g-(A(f)-A(g))}\leq 2\|f-g\|\leq 4r. 
\end{align*}
Hence, the following inequalities hold. 
\begin{align}
 2\mathrm{TV}(P_{g}, P_{f})
 &= 
 \int \abs{e^{f(x)-g(x)-(A(f)-A(g))}-1}p_{g}\dmu \nonumber\\
 &\geq 
 \frac{1}{8r}\int\abs{\<f-g,k_x-\mathbb{E}_{\beta}[k_x]\>} p_{g}\dmu \label{append:eqn:TV1}\\
 &\geq 
 \frac{\|f-g\|}{8r e^{2r}}\inf_{s\in\mathcal{H},\|s\|=1}
 \int \abs{s(x)-\mathbb{E}_{\beta}[s]}\dmu
 \label{append:eqn:TV2} \\
 &\geq
 \frac{\|f-g\|}{16r e^{2r}}\inf_{\substack{s\in\mathcal{H},\|s\|=1}}\int \abs{s(x)-\mathbb{E}_\beta[s]}^2\dmu \label{append:eqn:TV3} \\
 &\geq
 \frac{\|f-g\|}{16r e^{2r}}\inf_{\substack{s\in\mathcal{H},\|s\|=1}}
 \int \abs{s(x)-\mathbb{E}_{P_0}[s]}^2\dmu. \label{append:eqn:TV4} 
\end{align}
 \eqref{append:eqn:TV1} uses the inequality \eqref{appendix:exp-inequ}; 
 \eqref{append:eqn:TV2} is derived from the definition of $s(x)$ and the inequality $p_g\geq e^{-2r}p_0=e^{-2r}$; 
 \eqref{append:eqn:TV3} uses the inequality $\abs{x}\geq x^2/2$ for $\abs{x}\leq 2$; 
 \eqref{append:eqn:TV4} is derived from $\inf_{\alpha\in\Rbb}\mathbb{E}_{P_0}[\abs{s-\alpha}^2]=\mathbb{E}_{P_0}[\abs{s-\mathbb{E}_{P_0}[s]}^2]$. 
\end{proof}

\subsection{Proof of Corollary~\ref{thm:finite-dim-RKHS-minimax-opt}}
\label{appendix:subsec_minimax-opt-RKHS}
\begin{proof}
Since $\widehat{f}_r,f_0\in\mathcal{H}_{d,r}$, we  have $\|\widehat{f}_r-f_0\|\leq 2r$. From Lemma~\ref{lemma:lower-bd-TV_RKHS} and Theorem~\ref{thm:reg-STV-learning}, it holds that
\begin{align*}
 \|\widehat{f}_r-f_0\|\leq 
 \frac{32 r e^{2r}}{\inf_{d}\xi(\mathcal{H}_d)}
 \mathrm{TV}(P_{f_0}, P_{\widehat{f}_r})
 \lesssim 
  \frac{32 r e^{2r}}{\inf_{d}\xi(\mathcal{H}_d)} %\frac{32re^{2r}} {\mu_\infty}
\bigg(\varepsilon+\sqrt{\frac{d}{n}}\bigg)
\end{align*}
with high probability. 
Since the parameter $r$ is a positive constant, the upper bound in the RKHS norm is ensured. 
\end{proof}

\subsection{Proof of Theorem~\ref{thm:minimax-opt-NormalCov}}
\label{appendix:subsec_minimax-opt-NormalCov}

In the beginning, let us introduce a general theory of the minimax optimal rate of the parameter estimation. 
Let us define $\Theta\subset\Rbb^d$ be the $d$-dimensional parameter space of 
the statistical model $\{p_{\bm\theta}(x)\,:\,{\bm\theta}\in\Theta\}$. 
For any subset $T\subset\Theta$, 
let $\mathrm{vol}(T)$ be the volume of the subset $T\subset\Rbb^d$ for the usual Lebesgue measure. 
Define the KL diameter by 
$d_{\mathrm{KL}}(T)
=\sup_{{\bm\theta},{\bm\theta}'\in{T}}\mathrm{KL}(p_{\bm\theta}^{\otimes{n}},\,p_{\bm\theta'}^{\otimes{n}})
=n\sup_{\bm{\theta},{\bm\theta}'\in{T}}\mathrm{KL}(p_{\bm\theta},\,p_{\bm\theta'})$. 
For $B_\varepsilon=\{\bm{\theta}\in\mathbb{R}^d\,:\,\|\bm{\theta}\|_2\leq \varepsilon\}$, it holds that 
\begin{align}
 \label{appendix:eqn:general-lower-bound_Fano}
 \inf_{\widehat{\bm\theta}}\sup_{\bm{\theta}\in\Theta}P_{\bm\theta}(\|\widehat{\theta}-\theta\|\geq\varepsilon/2)
 \geq 1-\inf_{T\subset\Theta}\frac{d_{\mathrm{KL}}(T)+\log
 2}{\log\frac{\mathrm{vol}(T)}{\mathrm{vol}(B_{\varepsilon})}}. 
\end{align}
Detail is shown in \cite[Remark 3]{ma13:_volum}. 
For the contaminated distribution, we have the following theorem
\begin{theorem}[\cite{chen16:_huber}]
 $L$ is a loss function defined on the parameter space~$\Theta$. Define 
\begin{align*}
 \omega(\varepsilon,\Theta)=
\sup\left\{L({\bm\theta}_1,{\bm\theta}_2)\,:\,
 \mathrm{TV}(P_{{\bm\theta}_1},P_{{\bm\theta}_2})\leq\frac{\varepsilon}{1-\varepsilon},\,{\bm\theta}_1, {\bm\theta}_2\in\Theta
\right\}. 
\end{align*}
 Suppose there is some $\mathcal{R}$ such that 
 \begin{align*}
  \inf_{\widehat{\bm{\theta}}}  \sup_{{\bm\theta}\in\Theta} 
  P_{\bm\theta}\left(L(\widehat{\bm{\theta}}, \bm{\theta})\geq \mathcal{R}\right)>0. 
 \end{align*}
 Then, 
 \begin{align*}
  \inf_{\widehat{\bm{\theta}}}
  \sup_{\bm{\theta}\in\Theta, Q}
  P_{\bm{\theta},\varepsilon,Q}\left(L(\widehat{\bm{\theta}}, \bm{\theta})\geq
  \omega(\varepsilon,\Theta)+\mathcal{R}\right)>0, 
 \end{align*}
 where the supremum is taken over the distribution 
 $P_{\bm{\theta},\varepsilon,Q}=(1-\varepsilon)P_{\bm{\theta}}+\varepsilon Q$
 for any parameter $\bm{\theta}\in\Theta$ and any distribution $Q$. 
\end{theorem}

\begin{proof}
[Proof of Theorems~\ref{thm:minimax-opt-NormalCov}.]
 Note that $\|A\|_{\mathrm{op}}\leq \|A\|_{\mathrm{F}}$. 
 For $r<1/2\,(<R)$, let us define 
 $T_r:=I+B_{r}=\{I+A\,:\,\|A\|_{\mathrm{F}}\leq r\}\subset\mathfrak{T}_R=\{I+A\succ{O}\,:\,\|A\|_{\mathrm{op}}\leq R\}$. 
 For $\Sigma\in T_r$, we
 have $1+r\geq \sigma_1(\Sigma)\geq \sigma_d(\Sigma)\geq 1-r$. 
Hence, 
$\sigma_1(\Sigma_0^{-1}\Sigma_1)\leq \sigma_1(\Sigma_0^{-1})\sigma_1(\Sigma_1)\leq\frac{1+r}{1-r}$ 
and  $\sigma_d(\Sigma_0^{-1}\Sigma_1)\geq \sigma_d(\Sigma_0^{-1})\sigma_d(\Sigma_1)\geq\frac{1-r}{1+r}$
hold for $\Sigma_1, \Sigma_2\in T_r$. Therefore, we have 
\begin{align*}
 \frac{1}{3}\leq \frac{1-r}{1+r}\leq \sigma_k(\Sigma_0^{-1}\Sigma_1)\leq \frac{1+r}{1-r}\leq 3,\quad k=1,\ldots,d. 
\end{align*}
For $\Sigma_0, \Sigma_1\in T_r$, 
\begin{align*}
 \mathrm{KL}(N(\0,\Sigma_1),N(\0,\Sigma_0))
 &=\frac{1}{2}\sum_{i=1}^{d}\left(\sigma_i(\Sigma_0^{-1}\Sigma_1)-1-\log\sigma_i(\Sigma_0^{-1}\Sigma_1) \right)\\
 &\leq\frac{1}{2}\|\Sigma_0^{-1}\Sigma_1-I\|_{\mathrm{F}}^2\\
 &\leq\frac{1}{2}\|\Sigma_0^{-1}\|_{\mathrm{op}}^2\|\Sigma_1-\Sigma_0\|_{\mathrm{F}}^2\\
 &\leq \frac{1}{2(1-r)^2}\|\Sigma_1-\Sigma_0\|_{\mathrm{F}}^2. 
\end{align*}
Therefore, $d_{\mathrm{KL}}(T_r)\leq n\cdot\frac{4r^2}{2(1-r)^2}\leq 8n r^2$. 
Furthremore, $\mathrm{vol}(T_r)=C_{d}r^{d(d+1)/2}$ and $\mathrm{vol}(B_{\varepsilon})=C_{d}\varepsilon^{d(d+1)/2}$, 
where $C_d$ is a constant depending only on $d$. 
Hence, we have 
\begin{align*}
 \inf_{\widehat{\Sigma}}\sup_{\Sigma\in \mathfrak{T}_R}
 P_\Sigma(\|\widehat{\Sigma}-\Sigma\|_{\mathrm{F}}\geq\varepsilon/2)
 \geq 1-\frac{8nr^2+\log 2}{\frac{d(d+1)}{2}\log\frac{r}{\varepsilon}}
 \geq 1-\frac{16nr^2+2\log 2}{d^2\log\frac{r}{\varepsilon}}. 
\end{align*}
For $r=\sqrt{d^2/n}<1/2$ such that $n>4d^2$ and $\varepsilon=2^{-29}\sqrt{d^2/n}$, we have 
 \begin{align*}
 \inf_{\widehat{\Sigma}}\sup_{\Sigma\in \mathfrak{T}_R}
 P_\Sigma(\|\widehat{\Sigma}-\Sigma\|_{\mathrm{F}}\geq 2^{-30}\sqrt{d^2/n}) \geq 0.1. 
 \end{align*}
The modulus of continuity of the Frobenius norm is given by 
\begin{align*}
 \omega(\varepsilon,\mathfrak{T}_R)=\sup\{\|\Sigma_0-\Sigma_1\|_{\mathrm{F}}
 \,:\,\mathrm{TV}(P_{\Sigma_0}, P_{\Sigma_1})\leq \varepsilon/(1-\varepsilon),\,\Sigma_0,\Sigma_1\in\mathfrak{T}_R\}. 
\end{align*}
Pinsker's inequality yields that for $\Sigma_0, \Sigma_1\in T_{1/2}\subset\mathfrak{T}_R$, 
\begin{align*}
 \mathrm{TV}(N(\0,\Sigma_0),N(\0,\Sigma_1))\leq \sqrt{\frac{ \mathrm{KL}(N(\0,\Sigma_1),N(\0,\Sigma_0))}{2}}\leq \|\Sigma_0-\Sigma_1\|_F. 
\end{align*}
Hence, the pair of $\Sigma_0$ and $\Sigma_1$ such that 
$\|\Sigma_0-\Sigma_1\|_F=\varepsilon/(1-\varepsilon)\leq 1$ for $\varepsilon\leq 1/2$ leads to   
$\mathrm{TV}(N(\0,\Sigma_0),N(\0,\Sigma_1))\leq \varepsilon/(1-\varepsilon)$, 
meaning that $\omega(\varepsilon,\mathfrak{T}_R)\geq \varepsilon/(1-\varepsilon)\geq \varepsilon$. 
As a result, we have 
\begin{align*}
&\phantom{\geq} \inf_{\widehat{\Sigma}}
 \sup_{Q:\inf_{P\in\mathcal{N}_R}\mathrm{TV}(P,Q)<\varepsilon}
 Q\bigg(\|\widehat{\Sigma}-\Sigma\|_{\mathrm{F}} \geq \varepsilon+\sqrt{\frac{d^2}{n}}\bigg)\\
& \geq 
 \inf_{\widehat{\Sigma}}\sup_{\Sigma\in\mathfrak{T}_R, Q}
 P_{\Sigma,\varepsilon,Q}
 \bigg(\|\widehat{\Sigma}-\Sigma\|_{\mathrm{F}} \geq\varepsilon+\sqrt{\frac{d^2}{n}}\bigg)
 >0
\end{align*}
 for $\varepsilon\leq 1/2$. 
\end{proof}

\section{Proof of Theorem~\ref{theorem:error_rate-approximate-STV-learning}}
\label{app:Proof-thm:approximate-infinite-STV-learning-bound}

We prepare the following lemma. 
\begin{lemma}
 \label{thm:sampling-based_approximation}
 Suppose that $\sup_{x\in\mathcal{X}}k(x,x)\leq K^2$. 
 The Rademacher complexity of the function set 
 \begin{align*}
\{e^{f(z)}\sigma(u(z)-b)/q(z)\,:\,f\in\mathcal{H}_r, u\in\mathcal{H}_U,\, \abs{b}\leq{U}\}
 \end{align*}
is bounded above by 
\begin{align*}
 \frac{e^{Kr}}{\min_z q(z)}\left\{\mathfrak{R}_\ell(\mathcal{H}_r)+\mathfrak{R}_\ell(\mathcal{H}_U\oplus[-U,U])\right\}
 \lesssim
 \frac{e^{Kr}}{\min_z q(z)}\frac{r+U}{\sqrt{\ell}}. 
\end{align*}
 The Rademacher complexity of the function set $\{ e^{f(z)}/q(z)\,:\,f\in\mathcal{H}_r\}$
is bounded above by 
\begin{align*}
 \frac{e^{Kr}}{\min_z q(z)}\mathfrak{R}_\ell(\mathcal{H}_r)
 \lesssim
 \frac{e^{Kr}}{\min_z q(z)}\frac{r}{\sqrt{\ell}}. 
\end{align*}
\end{lemma}
\begin{proof}
 We compute the Rademacher complexity $ \mathfrak{R}_{Z}(\mathcal{F})$ of the following function set, 
$\mathcal{F}:=\{e^{f(z)}\sigma(u(z)-b)/q(z)\,:\,f\in\mathcal{H}_r, u\in\mathcal{H}_U, \abs{b}\leq{U}\}$. 
The proof is given by a minor modification of the proof in~\cite{mohri18:_found_machin_learn}. 
For completeness, we show the proof below. %We suppose $\sup_{x}k(x,x)\leq K^2$. 
In the below, $\gamma_1,\ldots,\gamma_\ell$ are i.i.d. Rademacher random variables. 
\begin{align*}
 &\phantom{=}\mathfrak{R}_{Z}(\mathcal{F})\\
& =
\frac{1}{\ell}\mathbb{E}
\bigg[ 
\mathbb{E}_{\gamma_{\ell}}\bigg[
 \sup_{f,u,b}\underbrace{\sum_{i=1}^{\ell-1}\gamma_i \frac{e^{f(z_i)}\sigma(u(z_i)-b)}{q(z_i)} }_{V_{\ell-1}(f,u,b)}
 +\gamma_\ell \frac{e^{f(z_\ell)}\sigma(u(z_\ell)-b)}{q(z_\ell)} 
 \bigg]\bigg]\\
& =
\frac{1}{\ell}\mathbb{E}
\bigg[ 
  \frac{1}{2}\bigg(V_{\ell-1}(f_{\ell,1},u_{\ell,1},b_{\ell,1}) 
 +\frac{e^{f_{\ell,1}(z_\ell)}\sigma(u_{\ell,1}(z_\ell)-b_{\ell,1})}{q(z_\ell)} \bigg)\\
& \phantom{=}\quad 
 + \frac{1}{2}\bigg(V_{\ell-1}(f_{\ell,2},u_{\ell,2},b_{\ell,2}) -\frac{e^{f_{\ell,2}(z_\ell)}\sigma(u_{\ell,2}(z_\ell)-b_{\ell,2})}{q(z_\ell)} \bigg)
\bigg]\\
& =
\frac{1}{\ell}\mathbb{E}
 \bigg[ 
 \frac{1}{2}\bigg(
 V_{\ell-1}(f_{\ell,1},u_{\ell,1},b_{\ell,1}) + V_{\ell-1}(f_{\ell,2},u_{\ell,2},b_{\ell,2})\\
&\quad 
 +
 \frac{e^{f_{\ell,1}(z_\ell)}\sigma(u_{\ell,1}(z_\ell)-b_{\ell,1})-e^{f_{\ell,1}(z_\ell)}\sigma(u_{\ell,2}(z_\ell)-b_{\ell,2})}{q(z_\ell)}\\
&\quad 
 +
 \frac{e^{f_{\ell,1}(z_\ell)}\sigma(u_{\ell,2}(z_\ell)-b_{\ell,2})-e^{f_{\ell,2}(z_\ell)}\sigma(u_{\ell,2}(z_\ell)-b_{\ell,2})}{q(z_\ell)}
 \bigg)\bigg]. 
\end{align*}
In the above, $f_{\ell,i}, u_{\ell,i}, b_{\ell,i},\, i=1,2$, are functions and bias terms that attain the supremum. 
We use the notations, $\ell_{1,\ell}:=\sign(u_{\ell,1}(z_\ell)-b_{\ell,1}-u_{\ell,2}(z_\ell)+b_{\ell,2})$ and 
$\ell_{2,\ell}:=\sign(f_{\ell,1}(z_\ell)-f_{\ell,2}(z_\ell))$. 
Then, the Lipschitz continuity ensures that the above equation is bounded above by 
\begin{align*}
&
\frac{1}{\ell}\mathbb{E}
 \bigg[ 
 \frac{1}{2}\bigg(
 V_{\ell-1}(f_{\ell,1},u_{\ell,1},b_{\ell,1}) + V_{\ell-1}(f_{\ell,2},u_{\ell,2},b_{\ell,2})\\
&\quad 
 +
 \frac{e^{K r}\ \ell_{1,\ell}(u_{\ell,1}(z_\ell)-u_{\ell,2}(z_\ell)-(b_{\ell,1}-b_{\ell,2}))}{\min_z q(z)}
 +
 \frac{e^{K r}\ \ell_{2,\ell} (f_{\ell,1}(z_\ell)-f_{\ell,2}(z_\ell))}{\min_z q(z)}
 \bigg)
\bigg]\\
& =
\frac{1}{\ell}\mathbb{E}
 \bigg[ 
 \frac{1}{2}\bigg( V_{\ell-1}(f_{\ell,1},u_{\ell,1},b_{\ell,1})  
+ \frac{e^{K r}}{\min_z q(z)}(\ell_{1,\ell} (u_{\ell,1}(z_\ell)-b_{\ell,1})+\ell_{2,\ell} f_{\ell,1}(z_\ell))\bigg)\\
&\qquad 
 + \frac{1}{2}\bigg( V_{\ell-1}(f_{\ell,2},u_{\ell,2},b_{\ell,2}) 
 - \frac{e^{K r}}{\min_z q(z)}(\ell_{1,\ell} (u_{\ell,2}(z_\ell)-b_{\ell,2})+\ell_{2,\ell} f_{\ell,2}(z_\ell))\bigg)\bigg]\\
& \leq 
\frac{1}{\ell}\mathbb{E}
 \bigg[ 
 \sup_{f,u,b}\frac{1}{2}\bigg( V_{\ell-1}(f,u,b)+\frac{e^{K r}}{\min_z q(z)}(\ell_{1,\ell} (u(z_\ell)-b)+\ell_{2,\ell} f(z_\ell))\bigg)\\
&\qquad 
 +
 \sup_{f,u,b}\frac{1}{2}\bigg( V_{\ell-1}(f,u,b)-\frac{e^{K r}}{\min_z q(z)}(\ell_{1,\ell} (u(z_\ell)-b)+\ell_{2,\ell} f(z_\ell))\bigg)\bigg]\\
& \leq 
\frac{1}{\ell}\mathbb{E}
 \bigg[ 
 \mathbb{E}_{\gamma_{\ell}} \bigg[ 
\sup_{f,u,b}\bigg( 
 V_{\ell-1}(f,u,b)+\frac{e^{K r}}{\min_z q(z)} \gamma_\ell((u(z_\ell)-b)+\ell_{1,\ell}\ell_{2,\ell}f(z_\ell))\bigg) 
 \bigg] \bigg]. 
\end{align*}
Repeating the same argument to the other terms in $V_{\ell-1}(f,u,b)$, we have
\begin{align*}
\mathfrak{R}_{Z}(\mathcal{F})
&\leq 
\frac{1}{\ell}\mathbb{E}
 \bigg[  \sup_{f,u,b}\sum_{i=1}^{\ell} \frac{e^{K r}}{\min_z q(z)} \gamma_i((u(z_i)-b)+\ell_{1,i}\ell_{2,i}f(z_i)) \bigg]\\
& \leq 
 \frac{1}{\ell}\mathbb{E}
 \bigg[  \sup_{u,b}\sum_{i=1}^{\ell} \frac{e^{K r}}{\min_z q(z)} \gamma_i(u(z_i)-b)\bigg]
+ \frac{1}{\ell}\mathbb{E}
 \bigg[  \sup_{f}\sum_{i=1}^{\ell} \frac{e^{K r}}{\min_z q(z)} \gamma_i f(z_i)\bigg] \\
&=
 \frac{e^{K r}}{\min_z q(z)}\mathfrak{R}_\ell(\mathcal{H}_U\oplus[-U,U])
+ \frac{e^{K r}}{\min_z q(z)}\mathfrak{R}_\ell(\mathcal{H}_r). 
\end{align*}
The second inequality of the above equations is obtained by the definition of the supremum. 
Note that the replacement of $\gamma_i\ell_{1,i}\ell_{2,i}$ with $\gamma_i$ does not change the expectation. 
Then, the first inequality of the lemma is obtained. 
The second inequality of the lemma is obtained from the standard Talagrand's lemma. 
\end{proof}

\begin{proof}
 [Proof of Theorem~\ref{theorem:error_rate-approximate-STV-learning}]
For the numerator and denominator of $\bar{\sigma}_\ell$ in \eqref{eqn:bar-sigma}, 
the uniform law of large numbers~\citep{mohri18:_found_machin_learn} with Lemma~\ref{thm:sampling-based_approximation} 
leads to the following uniform error bound, 
\begin{align*}
 &\sup_{\|f\|\leq{r},\|u\|\leq U,\abs{b}\leq U}\bigg\lvert\int e^{f(Z)}\sigma_{Z}\rmd\mu
 -\frac{1}{\ell}\sum_{j=1}^{\ell} \frac{e^{f(Z_j)}}{q(Z_j)}\sigma_{Z_j}\bigg\rvert  \\
 &\lesssim
 \frac{1}{\inf_z q(z)}
 \frac{2e^{Kr}(r+U)}{\sqrt{\ell}} + \frac{e^{Kr}}{\inf_z q(z)}\sqrt{\frac{\log(1/\delta)}{\ell}},\\
& \sup_{\|f\|\leq{r}}\bigg\lvert A(f)-\frac{1}{\ell}\sum_{j=1}^{\ell} \frac{e^{f(Z_j)}}{q(Z_j)}\bigg\lvert\\
 &\lesssim
 \frac{1}{\inf_z q(z)}
 \frac{2e^{Kr}r}{\sqrt{\ell}}+\frac{e^{Kr}}{\inf_z q(z)}\sqrt{\frac{\log(1/\delta)}{\ell}}, 
\end{align*}
where $\sigma_Z=\sigma(u(X)-b)$. 
By regarding $\inf_z q(z)$ as a constant, we have 
\begin{align*}
 \sup_{f\in\mathcal{H}_r,u\in\widetilde{\mathcal{H}}_U,\abs{b}\leq{U}}
 \abs{\mathbb{E}_{X\sim P_{f}}[\sigma(u(X)-b)]- \bar{\sigma}_\ell}
 \lesssim 
 \frac{e^{Kr}(r+U)}{\sqrt{\ell}}=:\xi 
\end{align*}
with high probability. Therefore, the estimation error of $\bar{\sigma}_\ell$ is 
\begin{align*}
 \sup_{f\in\mathcal{H}_r}
 \abs{\mathrm{STV}_{\widetilde{\mathcal{H}}_U,\sigma}(P_f,P_n)-\mathrm{STV}_{\widetilde{\mathcal{H}}_U,\sigma}(\widehat{P}_f,P_n)}
 & \leq
 \sup_{f\in\mathcal{H}_r}\mathrm{STV}_{\widetilde{\mathcal{H}}_U,\sigma}(P_f,\widehat{P}_f) \\
&\lesssim 
\xi. 
\end{align*}
As a result, we obtain
\begin{align*}
\mathrm{STV}_{\widetilde{\mathcal{H}}_U,\sigma}(P_{f_0}, P_{\widetilde{f}_r})
 &\leq
 \mathrm{STV}_{\widetilde{\mathcal{H}}_U,\sigma}(P_{f_0}, P_n)
 +
 \mathrm{STV}_{\widetilde{\mathcal{H}}_U,\sigma}(P_n, P_{\widetilde{f}_r})\\
& \lesssim 
 \mathrm{STV}_{\widetilde{\mathcal{H}}_U,\sigma}(P_{f_0}, P_n)
 +
 \mathrm{STV}_{\widetilde{\mathcal{H}}_U,\sigma}(P_n, \widehat{P}_{\widetilde{f}_r}) 
 + \xi\\
& \leq 
 \mathrm{STV}_{\widetilde{\mathcal{H}}_U,\sigma}(P_{f_0}, P_n)
 +
 \mathrm{STV}_{\widetilde{\mathcal{H}}_U,\sigma}(P_n, \widehat{P}_{\widehat{f}_r}) 
 + \xi\\
& \lesssim 
 \mathrm{STV}_{\widetilde{\mathcal{H}}_U,\sigma}(P_{f_0}, P_n)
 +
 \mathrm{STV}_{\widetilde{\mathcal{H}}_U,\sigma}(P_n, P_{\widehat{f}_r})
 + \xi\\
& \leq 
 2\mathrm{STV}_{\widetilde{\mathcal{H}}_U,\sigma}(P_{f_0}, P_n) + \xi\\
&
\lesssim
 \varepsilon +  \frac{C_{\mathcal{H}_U}}{\sqrt{n}} + \sqrt{\frac{\log(1/\delta)}{n}} + \xi. 
\end{align*}
Thus, we have
\begin{align*}
  \mathrm{TV}(P_{f_0}, P_{\widetilde{f}_r})
&=
 \left(\mathrm{TV}(P_{f_0}, P_{\widetilde{f}_r})
 -\mathrm{STV}_{\widetilde{\mathcal{H}}_U,\sigma}(P_{f_0}, P_{\widetilde{f}_r})
\right)
 +
 \mathrm{STV}_{\widetilde{\mathcal{H}}_U,\sigma}(P_{f_0}, P_{\widetilde{f}_r}) \\
&\lesssim
 \frac{r}{U} + \varepsilon +  \frac{C_{\mathcal{H}_U}}{\sqrt{n}} + \frac{e^{Kr}(r+U)}{\sqrt{\ell}} +
 \sqrt{\frac{\log(1/\delta)}{n}} + e^{Kr}\sqrt{\frac{\log(1/\delta)}{\ell}}
\end{align*}
with probability greater than $1-\delta$. 
\end{proof}




\end{appendices}


%\bibliography{allref}% common bib file



\end{document}